\documentclass{article}

\usepackage{PRIMEarxiv}

\usepackage[utf8]{inputenc} 
\usepackage[T1]{fontenc}    
\usepackage{hyperref}       
\usepackage{url}            
\usepackage{booktabs}       
\usepackage{amsfonts}       
\usepackage{nicefrac}       
\usepackage{microtype}      
\usepackage{lipsum}
\usepackage{fancyhdr}       
\usepackage{graphicx}       
\graphicspath{{media/}}     

\usepackage{array}
\usepackage{adjustbox}
\usepackage{longtable}
\usepackage{graphics}
\usepackage{tabularx}
\usepackage{xcolor}
\usepackage{multirow}
\usepackage{amsmath}
\usepackage{comment}
\usepackage{url}
\usepackage{makecell}
\usepackage{textcomp}
\usepackage{subfigure}
\usepackage{float}
\usepackage{booktabs} 
\usepackage{amsthm}
\usepackage{algorithm}

\usepackage[noend]{algpseudocode}

\newtheorem{defn}{Definition}[section]
\newtheorem{theorem}[defn]{Theorem}
\newtheorem{corollario}[defn]{Corollary}
\newtheorem{lemma}[defn]{Lemma}

\newtheorem{oss}[defn]{Observation}
\newenvironment{pf}{\medskip\noindent{\bf Proof.}\enspace}   {\hfill\newline\smallskip}

\title{ Empirical Density Estimation based on Spline Quasi-Interpolation with applications to Copulas clustering modeling

}

\author{
  Cristiano Tamborrino\\
  Department of Computer Science \\
  University of Bari \\
  Italy\\
  \texttt{cristiano.tamborrino@uniba.it} \\
   \And
  Antonella Falini \\
  Department of Computer Science \\
  University of Bari \\
  Italy\\
  \texttt{antonella.falini@uniba.it} \\
     \And
  Francesca Mazzia \\
  Department of Computer Science \\
  University of Bari \\
  Italy\\
  \texttt{francesca.mazzia@uniba.it} \\
}

\begin{document}
\maketitle

\begin{abstract}
 Density estimation is a fundamental technique employed in various fields to model and to understand the underlying distribution of data. The primary objective of density estimation is to estimate the probability density function of a random variable. 
This process is particularly valuable when dealing with univariate or multivariate data and is essential for tasks such as clustering, anomaly detection, and generative modeling. In this paper we propose the mono-variate approximation of the density using spline quasi interpolation and we applied it in the context of clustering modeling. 
The clustering technique used is based on the construction of suitable multivariate distributions which rely on the estimation of the monovariate empirical densities (marginals). Such an approximation is achieved by using the proposed spline quasi-interpolation, while the joint distributions to model the sought clustering partition is constructed with the use of copulas functions. In particular, since copulas can capture the dependence between the features of the data independently from the marginal distributions, a finite mixture copula model is proposed. The presented algorithm is validated on artificial and real datasets.

\end{abstract}

\keywords{ Spline Quasi-Interpolation, Copulas, Empirical Density Estimation, Clustering}



\section{Introduction}
Density estimation is a fundamental technique employed in various fields to model and understand the underlying distribution of data. It plays a pivotal role in capturing the inherent patterns and structures within a dataset, making it a crucial component in statistical modeling, machine learning, and data analysis.
The primary objective of density estimation is to estimate the probability density function (PDF) of a random variable. This process is particularly valuable when dealing with univariate or multivariate data and is essential for tasks such as clustering, anomaly detection, and generative modeling. Several methods have been developed to address the challenge of estimating the underlying density function. Classical approaches include histogram-based methods or kernel density estimation (KDE). These methods vary in complexity, assumptions, and computational efficiency, offering a range of options to suit different types of data and analytical requirements \cite{Rice_density, Rosemblatt56}. Other approaches in literature utilize splines \cite{spline_redner1992,spline_redner1994, splinedensitysurvey2023}. Understanding the strengths and limitations of each method is crucial for selecting an appropriate density estimation technique based on the characteristics of the data at hand. In this work, we aim to investigate a novel method for estimating the probability density through the utilization of a technique known as B-spline Hermite quasi-interpolation \cite{Mazzia2009}.  We will propose its application for the development of a new Copula-Based clustering algorithm, where density estimation plays a crucial role.
Clustering is considered an effective method for organizing data into  groups based on the similarities in features and characteristics among data points  \cite{JAIN2010651}.  In recent years, various algorithms have been developed for this purpose \cite{ABUALIGAH2018456, ZHOU2019546, Vine_SAHIN2022, CELEUX1992, dilascio, JAIN2010651, Kosmidis, Rajan}. For a comprehensive overview of the diverse approaches in the literature, e.g., \cite{Xu_Wunch_2005_review, Ezugwu_2022_review, Adil_2014_review,  SAXENA2017_review}. 
One of the most well-known and widely used approaches involves finite mixture modeling, which serves as a flexible and robust probabilistic tool for both univariate and multivariate data. However, it is important to note that Gaussian distributions, commonly used in finite mixtures, may not always accurately represent real-world data aggregation. To overcome this drawback, the use of alternative distributions, based on copulas, have gained significant interest in recent times for their potential to enhance the accuracy and applicability of clustering algorithms in a wider range of scenarios.

Copulas offer an alternative and increasingly popular approach to modeling data dependencies and aggregations \cite{Joe_1996, Nelsen2006, durante2015principles}. They provide a more versatile framework that does not rely on specific distributional assumptions like Gaussian distributions. Copulas can capture complex and non-linear relationships between variables, making them particularly useful for situations where data aggregation behaviors deviate from traditional distributions. The copulas have already been used as a possible solution to the clustering problem, for example, Di Lascio et all. \cite{dilascio}, address the problem of the clusterization by using different family of copulas. In \cite{Tewari}, \cite{Rajan2}, the clusterization process is done by the assumption that the multivariate dependencies is modelled by a new family of copula called Gaussian Mixture Copula, in \cite{Kosmidis} the authors analyze models of finite mixtures with different families of copulas, assuming that the marginals follow known distributions. 
The widely used method to fit the finite mixture model to observed data is the expectation maximization (EM) algorithm \cite{Dempster1977, bookEM_2007}, to estimate the maximum likelihood. Variations and/or adaptations to special situations of the EM algorithm exist, as the stochastic EM (SEM) algorithm (e.g., \cite{Celeux_1996, Marschner2001, Meng_1993}), the classification EM (CEM) algorithm (e.g.\cite{CELEUX1992}), the Monte CarloEM (MCEM) algorithm (e.g., \cite{Tanner_1987, Tanner_1990}) and those developed by \cite{Redner_1984, mclachlan200001}. 
An important point to take into account is that the choice of the copula model-based clustering imposes distributional assumptions on the marginals, along each dimension, and these marginal distributions are assumed or forced to be identical (e.g. a multivariate normal imposes univariate normal distribution on each marginal); such assumptions restrict the modeling flexibility. These restrictive assumptions could lead to erroneous modeling. To address this issue, a semiparametric approach is employed, in which marginal distributions are empirically estimated using kernel density estimation (KDE) \cite{Genest1995, joe2014dependence}. In this work, we propose a new strategy that enables the estimation of marginal densities through the use of the introduced B-Spline Quasi-Interpolants (BSHQI) and we describe a mixture model for density estimation based on Copulas that allows us to automatically choose a different Copula for each cluster. The paper is organized as follows, in Section \ref{sub:qispline}, the BSHQI density estimation is described together with its theoretical consistency properties and some statistical tests are conducted to validate the proved theoretical results. In Section \ref{sec:copMix} we revise some preliminary concepts related to copulas. In Section 4 the used EM algorithm for Copula Mixture models is detailed. In Section \ref{change_experiments} artificial and real datasets are analysed and finally some conclusive remarks are presented in Section \ref{sub:conclusion}.

\section{BSHQI density estimation}\label{sub:qispline}

Density estimation is a fundamental task in statistical analysis, involving the determination of the underlying probability distribution for a set of observed data. 

Let $X_1,\ldots, X_n$ be independent and identically distributed (i.i.d) random variables with an unknown Cumulative Distribution Function (CDF) $F(x)$. A non parametric estimator for $F(x)$ is provided by the Empirical Cumulative Distribution Function (ECDF)  $F_n(x) = \sum_{i=1}^n I(X_i \leq x)/n$,  where $I(X_i \leq x)$ is the indicator function, equal to 1 if $X_i \leq x$, and equal to 0 otherwise. 
However, the ECDF is discontinuous as
it jumps with size $1/n$ when $x=X_i, i =1\ldots, n$ and this is inconvenient since, often, the CDF itself is a continuous function.
The information given by the ECDF is the starting point to estimate the PDF.
Widely used methods to obtain an estimator of the PDF  are based on  Kernel Density Estimation (KDE) \cite{nadaraya1964,yamato1973,azzalini1981}. 
 KDE employs a kernel function, which serves as a ``base shape", to estimate the density of the data distribution.
There are various methods and kernels available for KDE, each with its own characteristics. The choice of kernel and method can significantly influence the accuracy of the density estimation. Some common kernels include the Gaussian kernel, the rectangular kernel, the Epanechnikov kernel, and others \cite{Rosemblatt56}. 
The most used  kernel function  is the uniform kernel:

\[
K(x) = \begin{cases}
    1 & \text{if }   x \in [-1/2,1/2], \\
    0 & \text{otherwise.}
\end{cases}
\]and the resulting  estimation of the density is called \textit{naive} kernel  \(\hat f_K(x)\),  \cite{Rice_density}:

\begin{equation}\label{naive}
    \hat f_K(x) = \frac{1}{nh} \sum_{i=1}^{n} K\left(\frac{X_i - x}{h}\right),
\end{equation}
where $h$ is the bandwidth parameter and its value affects the width of the KDE curves, and consequently the accuracy of the estimation.

We propose to estimate the CDF by applying the B-spline Hermite quasi-interpolant \cite{Mazzia2009}   to compute  $\hat{f}(x)$, approximation of  the probability density $f(x) = F'(x)$, 
so that the final approximation $\hat{F}(x)$ of the CDF  will be given by integrating $\hat{f}(x)$. 

Quasi-interpolation is a technique that allows to construct a local approximant by keeping low the computational cost, see e.g, \cite{dagnino2022spline,lyche1975local} and references therein. Generally, the common way to express a univariate spline quasi-interpolant (QI) of $d$-degree  reads as,
\begin{align}\label{eq:QI_change}
Q_d\;f(\cdot) = \sum_{j=-d}^{N-1}\lambda_j(f)B_{j,d}(\cdot),
\end{align}
where $B_{j,d}$ are $d$-degree B-splines assumed to be defined on an extended knot vector $
\tau:= \{\tau_{-d}, \ldots, \tau_{N+d}\},\,\tau_j \leq \tau_{j+1},
$
and spanning the space,
$
\mathbb{S}^{\pi}_d:=\langle B_{
-d,d},\ldots, B_{
N-1,d}\rangle.
$
The local linear functionals $\lambda_j$ in \eqref{eq:QI_change} can be computed by using several methodologies, such as differential, integral methods, and discrete approaches  see,e.g., \cite{de1973spline, lee2000some, lyche1975local,sablonniere2005recent}.
The main advantage of QI is that they have a direct construction without solving big
linear systems. Moreover, it is local, in the sense that the value of Q$f (x)$
depends only on values of $f$ in a neighbourhood of $x$. 

Given an interval $[a,b]$ such that $X_i \in [a,b]\; \mbox{ for } i=1,\ldots,n$  and a uniform mesh $\pi=\{a=x_0, x_1, \dots, x_N=b\}$ defined by a constant stepsize $h=(b-a)/N$. Note that the choice of $h$ is important and depends on $n$ as it plays the bandwidth role for the kernel density. 
For the estimation approximation of the PDF we use  the B-spline Hermite quasi-interpolant BSHQI defined in \cite{Mazzia2009}. BSHQI computes the $\lambda_j(f)$ as linear combination of the function $f$ and its derivatives evaluated at the mesh points.

We define the BSHQI with uniform knot vector $\pi$, coincident auxiliary knots and $d=2$. Hence, in the following to ease the notation, $B_{j,d} = B_j$. 
A discrete approximation of the sought CDF is expressed as,
  \begin{equation}\label{eq:discr_cumdistr}
F_{{h}}(x_j) := \frac{1}{n} \sum_{i=1}^n I(X_i \leq x_j),  \qquad  j=0,\ldots, N
\end{equation}

Starting from  $F_h(x)$ it is possible to approximate $f(x)$ at the mesh points by computing the first derivative  using finite differences:

\begin{equation}\label{approx_first_cdf_der}
\begin{array}{l}
   f(x_j) = F'(x_j)  \approx F'_{h,j} =\displaystyle{ \frac{F_h(x_{j+1}) - F_h(x_{j-1})}{2h}  }, \qquad j=1,\ldots,N-1 \\[0.3cm]
     f(x_0) = F'(x_0)  \approx F'_{h,0} = \displaystyle{ \frac{F_h(x_{1}) - F_h(x_{0})}{h} }, \qquad    f(x_N) = F'(x_N) \approx F'_{h,N} = \displaystyle{ \frac{F_h(x_{N}) - F_h(x_{N-1})}{h}. }
\end{array}
\end{equation}

Since the used quasi-interpolant is of Hermite type we  also need an approximation of the first derivative at the same mesh points:

\begin{equation}\label{approx_second_cdf_der}
\begin{array}{l}
   f'(x_j) = F''(x_j) \approx F''_{h,j} = \displaystyle{\frac{F_h(x_{j+1})-2F_h(x_j)+F_h(x_{j-1})}{h^2} }, \qquad j=1,\ldots,N-1 \\[0.3cm]
    f'(x_0) = F''(x_0) \approx   F''_{h,0} = 0, \qquad      f'(x_N) = F''(x_N) \approx F''_{h,N} = 0. 
\end{array}
\end{equation}

Note that the values attained by the density at $x_0$ and $x_N$ has been set to be zero, as this is the expected value of the sought CDF.


From the definition of the coefficients of BSHQI for $d=2$, we get

\begin{equation}\label{coeff_equation}
\begin{array}{l}
\lambda_j = \frac{1}{2}\left(F'_{h,(j+1)} + F'_{h,(j+2)}\right) - \frac{1}{4}h\left(-F''_{h,(j+1)} + F''_{h,(j+2)}\right),  \qquad j = -1, \ldots, N - 2 \\[0.3cm]
\lambda_{-2} = F'_{h,0}, \quad \lambda_{N-1}= F'_{h,N}.  \\
\end{array}
\end{equation}

The following theorem proves that the constructed $\hat{f}$ is indeed a continuous density function having first derivative continuous as well. 

\begin{theorem}\label{th:density}
The function  \(\hat{f}\), BSHQI estimation of  \(f\) in a given interval $[a,b]$,

\begin{equation}\label{dens_bsqi}
    \hat{f}(\cdot) = \sum_{j=-d}^{N-1}\lambda_j(f)B_{j}(\cdot),
\end{equation}
 with $\lambda_j$ as defined in \eqref{coeff_equation},
 is a density function. In particular:
 \begin{itemize}
 \item[a)]  $\hat{f}(\cdot)\geq 0$,
     \item[b)] $
    \int_{-\infty}^{+\infty}\hat{f}(x)\,dx = 1$,
    \item[c)] $\hat{f}\in C^1[a,b]$.
\end{itemize}
\end{theorem}

\begin{pf}

To prove a) it is sufficient to show that $\lambda_j\geq 0$, for $j=-d,\ldots,N-1$. In particular, since $h$ is constant, setting 
$$
F_j:= F_h(x_j),
$$
it can be shown that
  \begin{equation}\label{eq:coefL}
            \lambda_j=\frac{F_{(j+2)} - F_{(j+1)}}{h}, \qquad j=-1,\ldots,N-2 
\end{equation}
\[
\lambda_{-2}=\lambda_{-1} \qquad \lambda_{N-1}=\lambda_{N-2}.
\]
by substituting \eqref{approx_first_cdf_der} and \eqref{approx_second_cdf_der} into the equation \eqref{coeff_equation}. 
Therefore, it is straightforward to see that they
are  always positive.

To prove b), we set $\hat{f}(x)=0$ outside the interval $[a,b]$. Then, 
knowing that the integral of a B-spline is given by 

\[\int_a^b B_{i,d}(x) \, dx = 
\int_{\tau_i}^{\tau_{i+d+1}} B_{i,d}(x) \, dx = \frac{\tau_{i+d+1} - \tau_i}{d + 1},
\]
recalling that $d=2$, we have  
\begin{equation*}
   \tau_{i+d+1}-\tau_i= \left\{
   \begin{array}{lll}
      3h   &  \mbox{for } &i =0,\ldots,N-3, \\
      2h   &  \mbox{for } & i = -2,N-2,\\
      h   & \mbox{for } &i = -1, N-1.\\
    \end{array}\right.
\end{equation*}

Therefore,
\[
\int_{-\infty}^{+\infty} \hat{f}(x) \, dx = \int_{a}^{b} \sum_{j=-d}^{N-1}\lambda_j(f) B_{j}(x)\, dx = \sum_{j=-d}^{N-1}\lambda_j(f)\int_{\tau_j}^{\tau_{j+d+1}} B_{j}(x)\, dx =
\]
\[
\begin{split}
&=\frac{F_1-F_0}{h}\frac{h}{3}+\frac{F_1-F_0}{h}\frac{2h}{3}+\frac{F_2-F_1}{h}h+\cdots+\frac{F_{N-3}-F_{N-2}}{h}h+\frac{F_{N-1}-F_{N}}{h}\frac{2h}{3}+ 
\frac{F_{N-1}-F_{N}}{h}\frac{h}{3} \\ &= -F_0+F_N=0+1=1
\end{split}
\]

The point c) descends from the properties of the B-spline functions of degree $2$.

\qed
\end{pf}

In the following  we investigate the consistency of the derived density function following the analysis proposed in \cite{spline_redner1992}. In particular, the next Lemma will be useful as a preliminary result.

\begin{lemma}
The coefficients in (\ref{eq:coefL}) can be obtained by evaluating $\hat{f}_K(x)$ at $c_j := x_{j+1} + h/2, j=-1, \ldots, N-2$, i.e.:
\[
\lambda_j = \hat{f}_K(c_{j}) = \frac{1}{nh} \sum_{i=1}^{n} K\left(\frac{X_i - c_{j}}{h}\right), \qquad  j=-1,\ldots, N-2 
\]
and 
\[
\lambda_{-2} = \lambda_{-1}, \qquad  \lambda_{N-1} = \lambda_{N-2}, 
\]
\end{lemma}



\begin{theorem}\label{th:MSE}
   Let $X_1, \ldots, X_n$ denote i.i.d. observations having a PDF $f(x)\in C^1[a,b]$, $f$ and  $f^{\prime}$ bounded, and let $B_{j}(x)$ denote the $j-th$, $2$-nd degree B-spline basis. Let $x\in [a,b]$ and  let $\hat{f_n}(x)$ be as in (\ref{dens_bsqi}), with  uniform mesh defined in $[a,b]$ choosing a constant $h$ such that  as $n \to \infty$, as $nh \to \infty$ and as $h \to 0$, then  $\hat{f}_n$ is a uniformly consistent estimator of $f$.



\end{theorem}
\begin{proof}
To prove the pointwise consistency of $\hat{f}_n$ it is necessary to show that the $MSE (\hat{f}_n(x)) \rightarrow 0$  as $n \to \infty$, as $nh \to \infty$ and as $h \to 0$. In the following we use the well-known "Bias-Variance"  formulation: 

\[MSE (\hat{f}_n(x)) \equiv E\left[\left|\hat{f}_n(x) - f(x)\right|^2\right] = \text{Var}(\hat{f}_n(x)) + \left[E(\hat{f}_n(x)) - f(x)\right]^2 \equiv \text{Var}(\hat{f}_n(x)) + \text{Bias}^2(\hat{f}_n(x)).
\]
We start considering the absolute value of the Bias:
\[
\begin{split}
\left|\text{Bias}(\hat{f}_n(x))\right|&= \left|E\left(\frac{1}{nh} \sum_{j} \lambda_j B_{j}(x) \right)- f(x)\right| \\
&= \left|E\left(\frac{1}{nh} \sum_{j}\sum_{i=1}^{n} K\left(\frac{X_i - c_j}{h}\right)B_{j}(x)\right)- f(x) \right|\\
&= \left|\frac{1}{h}\sum_j B_{j}(x) E\left(K\left(\frac{X - c_j}{h}\right)\right)- f(x)\right|  \\
&=  \left|\frac{1}{h}\sum_j B_{j}(x)\left(\int f(X) K\left(\frac{X- c_j}{h}\right)dX\right) - f(x)\left(\int\frac{1}{h}\sum_{j} K\left(\frac{X - c_j}{h}\right)B_{j}(x) dX\right) \right|
\end{split}
\]
\text{since} $\displaystyle{\int\frac{1}{h}\sum_{j} K\left(\frac{X - c_j}{h}\right)B_{j}(x) dX=1 } $ \mbox{ and  denoting with  } $I_{x,h}=[x_{j+1},x_{j+2}] = supp\{B_j(x)\}\cap supp\{K\left( \frac{X-c_j}{h}\right)\}$ \mbox{ we have, }
\\

\[
\begin{split}
 &=  \left|\frac{1}{h}\sum_j B_{j}(x)\left(\int f(X)K\left(\frac{X- c_j}{h}\right)dX\right) - f(x)\frac{1}{h}\sum_{j}B_{j}(x)\int K\left(\frac{X - c_j}{h}\right) dX \right|\\ 
 &=\left|\frac{1}{h}\sum_j B_{j}(x)\left(\int (f(X)-f(x))K\left(\frac{X- c_j}{h}\right)dX\right)\right|\\
&\leq \sup_{X\in I_{x,h}}|(f(X)-f(x)| \frac{1}{h} \sum_jB_{j}(x) \int K\left(\frac{X - c_j}{h}\right) dX =\sup_{X\in I_{x,h}}|f(X)-f(x)|\\
&\leq \sup_{\xi\in I_{x,h}}|f'(\xi)|h\\
\end{split}
\]
 therefore 
\[
\text{Bias}^2\left(\hat{f}_n(x)\right)\leq h^2\left(\sup_{\xi \in I_{x,h}}|f'(\xi)|\right)^2.
\]

Considering now the variance we deduce that:
\[
\begin{split}
\text{Var}(\hat{f}_n(x)) &= \text{Var}\left(\sum_{j} \lambda_j B_{j}(x)\right)= \\
& = \text{Var}\left(\frac{1}{nh} \sum_{j}\sum_{i=1}^{n} K\left(\frac{X_i - c_j}{h}\right)B_{j}(x)\right)= \\
&=\mathbf{E}\left[\left(\frac{1}{nh} \sum_{j}\sum_{i=1}^{n} K\left(\frac{X_i - c_j }{h}\right)B_{j}(x)\right)^2\right]-
\left(\mathbf{E}\left[\frac{1}{nh} \sum_{j}\sum_{i=1}^{n} K\left(\frac{X_i - c_j}{h}\right)B_{j}(x)\right]\right)^2\\
&=\frac{1}{nh^2}\sum_j\sum_zB_{j}(x)B_{z}(x)\int K\left(\frac{X - c_j}{h}\right)K\left(\frac{X - c_z}{h}\right)f(X)dX-\\
&\frac{1}{nh^2}\sum_j\sum_zB_{j}(x)B_{z}(x)\int K\left(\frac{X - c_j}{h}\right)f(X)dX\int K\left(\frac{X - c_z}{h}\right)f(X)dX\\
&=\frac{1}{nh^2} \sum_j\sum_zB_{j}(x)B_{z}(x)\\
&\left( \int K\left(\frac{X - c_j}{h}\right)K\left(\frac{X - c_z}{h}\right)f(X)dX-\int K\left(\frac{X- c_j}{h}\right)f(X)dX\int K\left(\frac{X- c_z}{h}\right)f(X)dX\right)\\
\end{split}
\]

since the kernel $K$ is evaluated at the mesh points it is easy to prove that this quantity is bounded and  $$  h  =\int K\left(\frac{X - c_j}{h}\right)K\left(\frac{X - c_z}{h}\right)dX, $$\\
hence,
\[
\begin{split}
\text{Var}(\hat{f}_n(x)) = &\leq\frac{1}{nh^2} \sum_j\sum_zB_{j}(x)B_{z}(x)
\left(\sup_{X \in \mathbb{R}} |f(X)| h +\sup_{X \in \mathbb{R}} (f(X))^2\left( \int K\left(\frac{X- c_j}{h}\right)dX\right)\left(\int K\left(\frac{X- c_z}{h}\right)dX\right)\right)\leq\\
&\leq\frac{1}{nh^2} \sum_j\sum_zB_{j}(x)B_{z}(x)\left(\sup_{X \in \mathbb{R}} |f(X)| h+\sup_{X \in \mathbb{R}} (f(X))^2h^2\right)\\
&=\frac{1}{nh}\left(\sup_{X \in \mathbb{R}} |f(X)| +\sup_{X \in \mathbb{R}} (f(X))^2h\right)\left(\sum_jB_{j}(x)\right)\left(\sum_zB_{z}(x)\right)\\
&=\frac{1}{nh}\sup_{X \in \mathbb{R}} |f(X)| +\frac{1}{n}\sup_{X \in \mathbb{R}} (f(X))^2.
\end{split}
\]
\end{proof}
Then we have
\begin{equation}
\text{MSE}(\hat{f}_n(x)) =\text{Bias}^2(\hat{f}_n(x)) + \text{Var}(\hat{f}_n(x)) \leq h^2\left(\sup_{\xi \in I_{x,h}}|f'(\xi)|\right)^2 + \frac{1}{nh}\sup_{X \in \mathbb{R}} |f(X)| +\frac{1}{n}\sup_{X \in \mathbb{R}} (f(X))^2.
\end{equation}
This quantity will tend to zero as as $n \rightarrow \infty$, as $nh \rightarrow \infty$, and as $h \rightarrow 0$.  

Since $\displaystyle{\sup_{\xi \in I_{x,h}}|f'(\xi)|\leq S_1} :=\displaystyle{ \sup_{\xi\in [a,b]} |f^\prime(\xi)|}$, denoting by $\displaystyle{S_0 := \sup_{X\in [a,b]}|f(X)|}$, then, the upper bound for the MSE can be written as:
$$
MSE(\hat{f}_n(x))\leq \text{Bias}^2(\hat{f}_n(x)) + \text{Var}(\hat{f}_n(x)) \leq h^2(S_1)^2 + \frac{1}{nh}S_0 + \frac{1}{n}S_0^2.
$$

The above upper bound does not depend on $x$ as so the uniformly consistency is proved. \qed

In the BSHQI density estimation, the bandwidth $h$ can be freely chosen as long as the assumption of Theorem \ref{th:MSE} are satisfied.  Thus, the optimal bandwidth $h$ can be chosen by minimizing the MSE neglecting the smallest term $\frac{1}{n}S_0^2$:

\begin{equation}
h_{\text{opt}}(x) := \left(\frac{1}{2n} \frac{S_0}{(S_1)^2}\right)^{\frac{1}{3}}\sim n^{-1/3}.
\end{equation}
This choice for the smoothing bandwidth leads to a MSE at the rate

\begin{equation}
\text{MSE}_{\text{opt}}(\hat{f}_n(x)) = O(n^{-\frac{2}{3}}).
\end{equation}
In the previous analysis, our focus was solely on a single point, $x$. However, in a broader context, our goal is to manage the overall MSE  for every point. In such cases, a straightforward extension is the mean integrated square error (MISE) of $\hat{f}_n(x)$. We have the following corollary:

\begin{corollario}
 Let $f$ be a probability density function on $\mathbb{R}$, and let $A \subseteq \mathbb{R}$ be an open region with $A = \{ x \mid f(x) \neq 0 \}$, where $f \in C^1(\mathbb{R})$, and both $f$ and $f'$ are bounded. If $A$ is contained in a closed and bounded region, then, with $\hat{f}_n$ as defined above, if $h \to 0$ and $nh \to \infty$ as $n \to \infty$, then
\[ \text{MISE}(\hat{f}_n) = \int \text{MSE}(\hat{f}_n(x)) \, dx= \int \mathbf{E}\left(\hat{f}_n(x) - f(x) \right)^2dx \to 0, \]
as $n \to \infty$, i.e., $\hat{f}_n$ is a consistent estimator of $f$ in mean integrated squared error (MISE).

\end{corollario}
The above evidence concludes the investigation of the consistency results of the B-spline estimator with the proposed approach. However, in this context, how to choose $h$ is an unsolved problem in statistics known as \textit{bandwidth selection}. Most bandwidth selection approaches either suggest an estimate of AMISE and then aim to minimize the estimated AMISE. For more details, we refer to \cite{spline_richard99}.

\begin{oss}
Given the weighted CDF
\begin{equation}\label{eq:weightedCDF}
F_{{h,w}}(x_j) = {\frac{1}{\sum_{i=1}^n w_i}} \sum_{i=1}^n w_i I(X_i \leq x_j),  \qquad  j=0,\ldots, N,
\end{equation}  
with $w_i$  a weight associated to the corresponding observation $X_i$, then the results of Theorem \ref{th:density} and Theorem \ref{th:MSE} continue to hold.
\end{oss}
\begin{oss}
Note that the continuous CDF $\hat{F}$ is then computed by integrating the density $\hat{f}$ in Equation \eqref{dens_bsqi}.
\end{oss}

\subsection{Statistical Tests for marginals fitting with BSHQI spline}\label{sub:qispline_test}
In this subsection we compare the BSHQI density estimation  with 
the classical empirical approach constructed by using the Gaussian Kernel Density function. The careful selection of  the number of bins is of crucial importance in approximating density, as it directly influences the accuracy and the visual representation of the data distribution pattern. For both procedures considered in this work, it is possible to select different criteria for choosing the  number of bins, in particular we consider, the so called \textit{Rice’s Rule} \cite{wand}, where the number of bins is equal to $2 \times \lceil n^{1/3}\rceil$.

 There is no single optimal criterion for selecting the most suitable bins. For the experiments conducted in this work, unless otherwise indicated, we will use the Rice rule.
 To assess the goodness of the produced model, we conduct two statistical tests: the Kolmogorov-Smirnov (KS) Test \cite{kolmogorov_1951} and the Cramér-von Mises (CvM) Test \cite{Cramer_1962}. Additionally, we show the error in terms of Average Mean Integrated Squared Error (AMISE) and Root Mean Square Error (RMSE) for the computed probability density functions.

We perform the tests on three different distributions: a normal distribution $X\sim \mathcal{N}(\mu,\sigma^2)$ with $\mu=5$ and $\sigma^2=0.3$, an exponential distribution $X \sim \text{Exp}(\lambda)$ with $\lambda=1$ and a third distribution consisting of a mixture of Gaussians with different means and variances.\\
\begin{figure}[htbp]
		\centering
		\subfigure[Pdf]{\includegraphics[width=6cm]{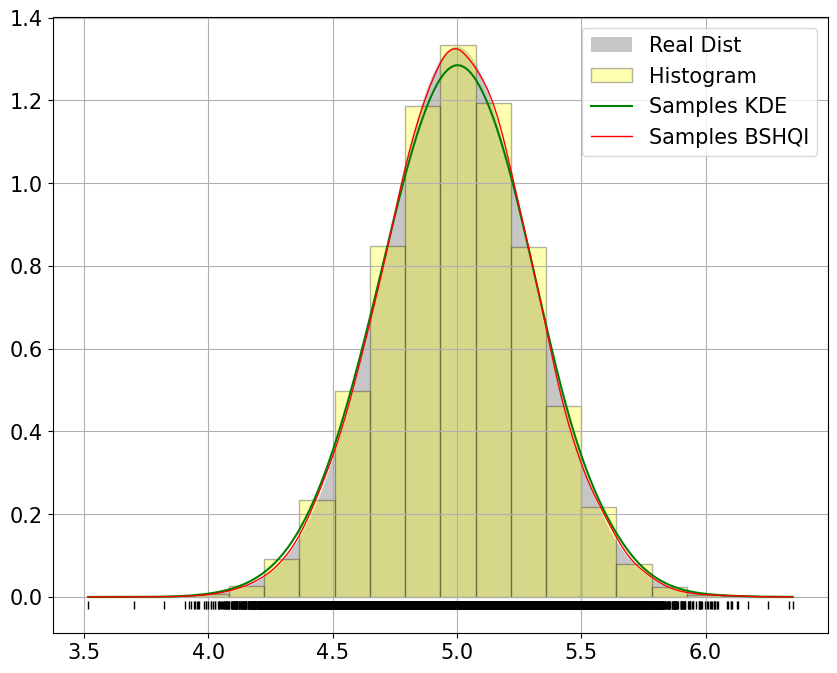}}\quad
		\subfigure[Cdf]{\includegraphics[width=6cm]{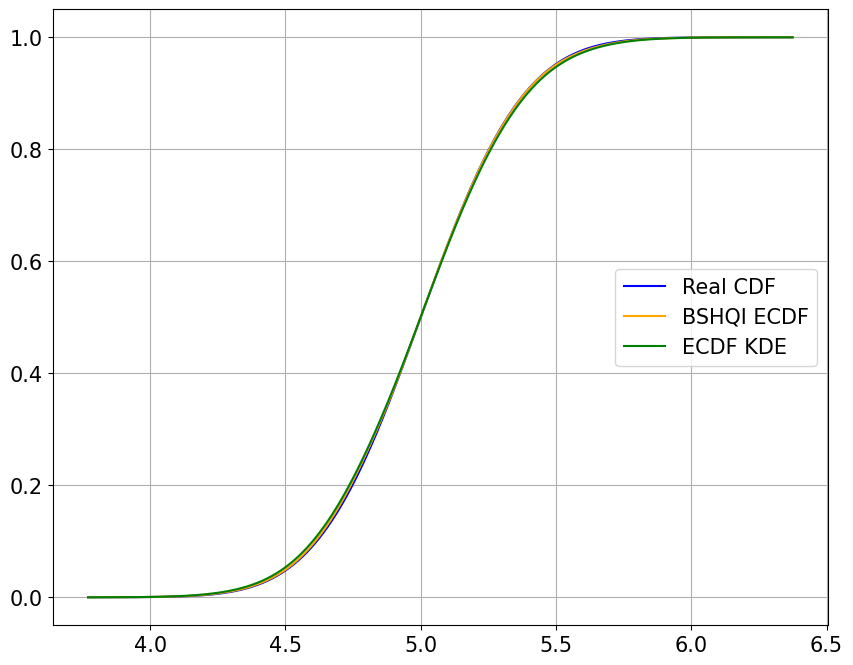}}\quad
  \caption{Comparison of samples generated from $X\sim \mathcal{N}(5,0.3)$ with the KDEpy and BSHQI method for probability density (a) and for the cumulative distribution (b).}
  \label{fig:stat_test_norm}
  \end{figure}
\begin{table}[htbp]
    \centering
    \begin{adjustbox}{max width=1\textwidth}
    \begin{tabular}{ccc|rrrr}
    \hline
    & AMISE &RMSE &\multicolumn{2}{c}{KS-Test} & \multicolumn{2}{c}{Cramér–von Mises} \\
     \hline
   & &       & \texttt{statistic} & \texttt{p-value} & \texttt{statistic} & \texttt{p-value} \\
     \hline
    \hline    &          &          &           &          &                    &          \\
 BSHQI     & 3.43e-06 & 1.13e-04 & 7.75e-03  & 2.77e-01 & 2.05e-01           & 2.58e-01 \\
 \hline    &          &          &           &          &                    &          \\
 KDEpy & 1.08e-05 & 3.53e-04 & 1.21e-02  & 1.58e-02 & 9.08e-01           & 4.05e-03 \\

\hline 

    \end{tabular}
    \end{adjustbox}
    \caption{Statistics ran on the results for Normal Density Estimation--- Rice's Rule for bins.}    
    \label{tab:NormTest}
\end{table}
All the numerical experiments are performed using Python 3.10 on a computing system equipped with Windows 11 operating system, 16 GB of RAM, and powered by an Intel(R) Core(TM) i7-9750H CPU @ 2.60GHz with a base clock speed of 2.59 GHz. The python package KDEpy\footnote{\url{https://kdepy.readthedocs.io/en/latest/introduction.html}} has been chosen for the comparison since, in our opinion, it is the most efficient  among the available Python routines for this task;  the Kernel density estimation is constructed following the theory in \cite{lauter1988silverman}. 
\begin{figure}[htbp]
		\centering
		\subfigure[Pdf]{\includegraphics[width=6cm]{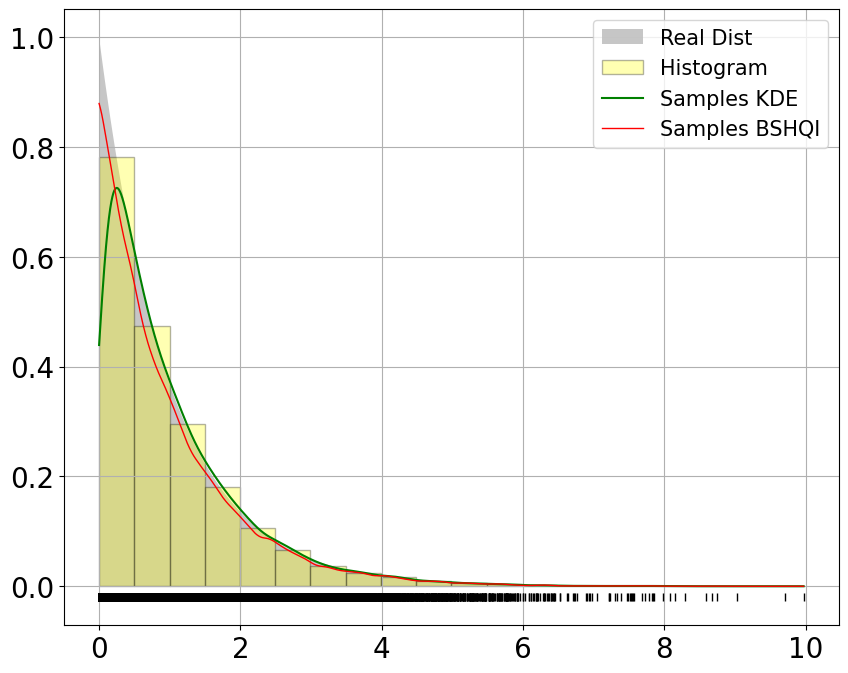}}\quad
  \subfigure[Cdf]{\includegraphics[width=6cm]{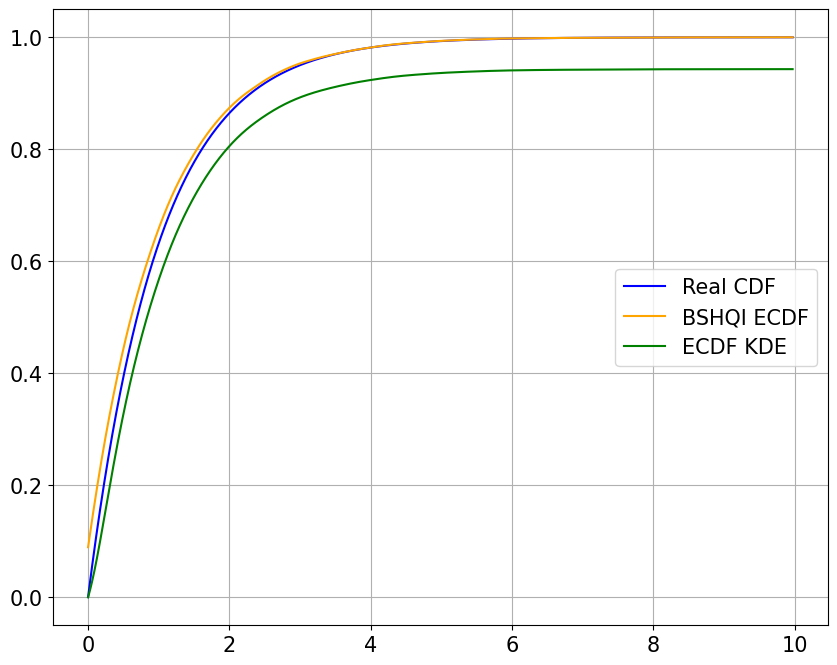}}
		\caption{Comparison of samples generated with the KDE and BSHQI method for probability density (a) and for the cumulative distribution (b)}
		\label{fig:stat_test_exp}
\end{figure}

\begin{table}[h]
    \centering
    \begin{adjustbox}{max width=1\textwidth}
    \begin{tabular}{ccc|rrrr}
    \hline
    & AMISE &RMSE &\multicolumn{2}{c}{KS-Test} & \multicolumn{2}{c}{Cramér–von Mises} \\
     \hline
   & &       & \texttt{statistic} & \texttt{p-value} & \texttt{statistic} & \texttt{p-value} \\
     \hline
    \hline    &          &          &           &          &                    &          \\
 EMP\_BSHQI     & 2.27e-06 & 2.96e-05 & 7.78e-03  & 2.73e-01 & 1.18e-01           & 5.02e-01 \\
 \hline    &          &          &           &          &                    &          \\
 EMP\_KDEpy & 1.68e-04 & 2.19e-03 & 6.96e-02  & 0 & 5.79e+01           & 2.01e-08 \\
\hline
\hline
    \end{tabular}
    \end{adjustbox}
    \caption{Statistics ran on the results for Exponential Density Estimation--- Rice's Rule for the bins.}   
    \label{tab:ExpTest}
\end{table}
For all the considered distributions, we generated a group of $n=2^{15}$ samplings for $20$  iterations. This iterative process allows us to calculate the AMISE, the RMSE, the values of both statistics and relative p-values, and we derive their average values as comprehensive measures of performance. Furthermore, we evaluated efficiency in terms of computational time by calculating the mean and standard deviation after the set number of iterations. Regarding the statistical tests, the null hypothesis states that the true underlying distribution and the empirical one are identical; the alternative hypothesis suggests that they are not. The used statistics is the maximum absolute difference between the exact values and the ones computed by the 
empirical distribution functions at the same  samples. If the KS or CvM  statistics are large, then the p-value will be small, and this may be taken as evidence against the null hypothesis in favor of the alternative.

By observing the Tables \ref{tab:NormTest}, \ref{tab:ExpTest}, \ref{tab:MixNormTest} it can be seen that the density approximation with BSHQI is preferable to the classical empirical evaluation of the distribution taken into account.
\begin{figure}[htbp]
\centering
  \subfigure[Pdf]{\includegraphics[width=6cm]{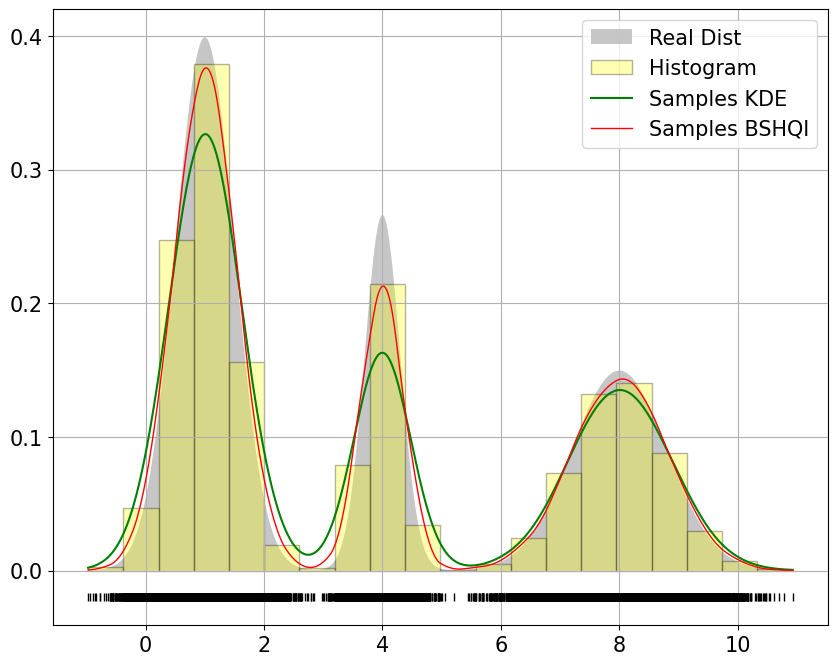}}\quad
  \subfigure[Cdf]{\includegraphics[width=6cm]{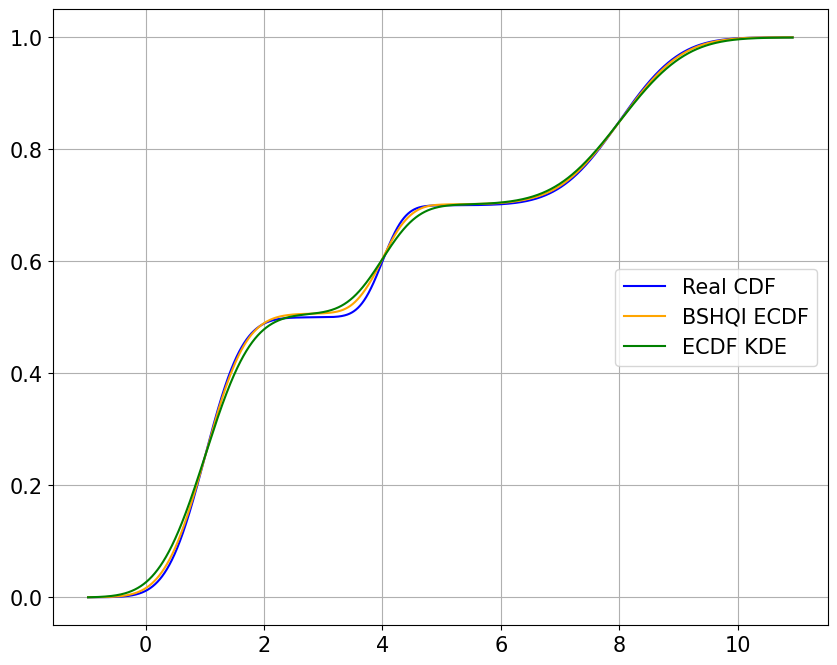}}\quad
		\caption{Comparison of samples generated with the KDE and BSHQI method for probability density (a) and for the cumulative distribution (b).}
		\label{fig:stat_test_mix}
\end{figure}

\begin{table}[htbp]
    \centering
    \begin{adjustbox}{max width=1\textwidth}
    \begin{tabular}{ccc|rrrr}
    \hline
    & AMISE &RMSE &\multicolumn{2}{c}{KS-Test} & \multicolumn{2}{c}{Cramér–von Mises} \\
     \hline
   & &       & \texttt{statistic} & \texttt{p-value} & \texttt{statistic} & \texttt{p-value} \\
     \hline
    \hline    &          &          &           &          &                    &          \\

 EMP\_BSHQI     & 1.84e-06 & 2.16e-05 & 4.91e-03  & 8.22e-01 & 8.09e-02           & 6.87e-01 \\
 \hline    &          &          &           &          &                    &          \\
 EMP\_KDEpy & 9.35e-06 & 1.10e-04 & 1.06e-02  & 5.04e-02 & 5.18e-01           & 3.58e-02 \\
\hline

    \end{tabular}
    \end{adjustbox}
    \caption{Statistics ran on the results for Mixture Gaussian Density Estimation--- Rice's Rule for the bins.}
  
     \label{tab:MixNormTest}
\end{table}

Indeed, comparing the p-values of both tests in each of the Tables \ref{tab:NormTest}, \ref{tab:ExpTest}, \ref{tab:MixNormTest}, allows us to accept the hypothesis when using the BSHQI, contrary to what can be concluded when referring to  KDEpy.

Moreover, for all three experiments, the AMISE and RMSE obtained with the proposed approach are lower compared to the AMISE and RMSE obtained using KDEpy. This can be observed in Figures \ref{fig:stat_test_norm}, \ref{fig:stat_test_mix}, and \ref{fig:stat_test_exp}, where the estimates of the PDF and the CDF for both methods are compared to the real distribution. Note that different results could be obtained by choosing a different rule for the number of bins. In particular, for the normal distribution in Figure \ref{fig:stat_test_norm} we can observe a good agreement for both the methodologies; for the exponential and the mixture ones in Figures \ref{fig:stat_test_mix}, and \ref{fig:stat_test_exp}, the result produced by BSHQI seems to better alienate with the original distribution while KDE seems to be less accurate as it always produce an under-estimate of the original distribution. The visual output is also supported by the results shown in Tables \ref{tab:NormTest},\ref{tab:ExpTest},\ref{tab:MixNormTest}: if on the one hand the reached accuracy seems similar in terms of AMISE and RMSE, on the other hand, the statistics conducted to test the validity of the null hypothesis clears out any doubt in assessing a better goodness of fit, under statistical point of view, of the BSHQI method. 
\begin{table}[htbp]
    \centering
    \begin{tabular}{lcc}
        \toprule
        \textbf{Algorithm} & \textbf{Mean Time (ms)} & \textbf{Standard Deviation (ms)} \\
        \midrule
        KDEpy & $0.0241$ & $\pm 0.0011$ \\
        BSHQI & $0.0124$ & $\pm 0.0009$ \\
        \bottomrule
    \end{tabular}
        \caption{Comparison of Execution Times for KDE and BSQH Algorithms}\label{tab:table_time}
\end{table}
In addition examining the results in Table \ref{tab:table_time}, it can be seen an advantage in terms of computational time when utilizing the BSHQI, compared to the alternative approach. While the difference is not substantial, it is sufficient to highlight the efficiency of our approach as the time is almost halved when using BSHQI.  
\section{Copulas Mixture Model}\label{sec:copMix}

In this section, we leverage the density estimation strategy introduced in Section \ref{sub:qispline} in the context of copulas. Our objective is to introduce the concept of copula and emphasize its profound connection to marginal distributions, with a specific focus on its application in formulating a novel clustering algorithm based on copula mixture.
Copulas functions are a useful tool employed to easily express multivariate distributions by specifying the marginals. In this Section the principal concepts are revised following the setting in \cite{durante2015principles}.

\begin{defn}\label{def:C}
	A D-dimensional copula is a CDF with
	uniform marginals: 
 $$C : [0, 1]^D \rightarrow [0, 1] \;\mbox {such that } C(u) = C(u_1,\ldots, u_D).$$
\end{defn}

\begin{theorem}{\textbf{(Sklar's Theorem)}}\label{Skltheorem}
Consider a D-dimensional CDF, $G$, with marginals $F_1,\ldots,F_D$. Then there exist a copula, $C$, such that
	\begin{equation}\label{eq:sklar}	G(x_1,\ldots,x_D)=C(F_1(x_1),\ldots,F_D(x_D))
	\end{equation}
	for all $x_i\in[-\infty,\infty]$ and $i=1,\ldots,D$.
	If $F_i$ is continuous for all $i=1,\ldots,D$ then $C$ is unique.
 In the opposite direction, given a copula $C$ and univariate CDFs, $F_1,\ldots,F_D$, then $G$ as in \eqref{eq:sklar} is a multivariate CDF with marginals $F_1,\ldots ,F_D.$ 
\end{theorem}

Given a CDF $G$ with PDF $g$, and a copula $C$ defined as in Definition \ref{def:C}, the density copula function $c$ can be computed as,

\begin{equation}\label{density}
c(u_1,\ldots,u_D)=\frac{g(F_1^{-1}(u_1),\ldots,F_D^{-1}(u_D))}{f_1(F_1^{-1}(u_1))\cdots f_D(F_D^{-1}(u_D))},
\end{equation}
where  $f_1, \ldots, f_D$ are the PDFs of the marginals. 

A novel algorithm, which takes advantage of the proposed construction for empirical cumulative distribution based on BSHQI, for estimating the marginals of a chosen copula is presented. As main application, we show the performance of such algorithm in the clustering context.  Therefore, the following formulation will be framed within the clustering setting. 

Our goal is to implement a model-based algorithm capable of correctly identifying how the instances of the dataset can be grouped into different clusters. In this sense, using copulas provides a way to fit data that have different probability distributions, thus having an advantage in better discriminating the possible clusters of a dataset.
We can assume an a priori model made with $K$ clusters and the data in each single cluster are distributed as a multidimensional Copula belonging to the Elliptical family, in particular Gaussian Copula, and Archimedean family, in which we consider the Clayton, Gumbel and Frank copulas \cite{Joe_1996}.  
We therefore stress the fact that our mixture distribution is composed of a linear combination of Copulas. This linear combination is called ``\textit{Copula Mixture Model}".  In the following we describe in details the derived formulation. 


\begin{defn}[Semiparametric approach]\label{eq:mixC}
A Copula Mixture is a function consisting of several Copula density functions $c_k$, with $k\in\{1,\ldots,K\}$ and $K$ denoting the number of clusters of the considered dataset. Each Copula $c_k$ in the mixture is characterized by a vector $\boldsymbol{\omega}$ that defines the parameters of the specific copula chosen for the mixture, and by the methods chosen for the approximation of the marginals. Moreover, for each Copula density function $c_k$ is defined a mixing probability $\pi_k$, referred to as mixing coefficient, such that:
\[
\sum_{k=1}^{K}\pi_k=1.
\]
\end{defn}
Let us assume a dataset \(\mathbf{X} = (X_1, X_2, \ldots, X_D)\) where each $X_i$ consists of \(n\) i.i.d. observations, and where the \(i\)-th observation is  \(\mathbf{x}_i = (x_{i,1}, x_{i,2}, \ldots, x_{i,D})^T\) for $i=1\ldots,n$. 

We express the probabilistic model of mixture copulas for all observations in the following form:

		\begin{equation}\label{join_X_prob}
		    	p(\mathbf{X}|\boldsymbol{\theta})=\prod_{i=1}^{n}p(\mathbf{x}_{i})=\prod_{i=1}^{n}\sum_{k=1}^{K}\pi_k g_k(\mathbf{x}_{i}|\boldsymbol{\omega}_k)
		\end{equation}
where $p$ is the probability, $\boldsymbol{\theta}:=\{\pi_k,\boldsymbol{\omega}_k\}$ indicates the parameters of the model: $\pi_k$ is the mixing probability and $\boldsymbol{\omega_k}$ represents the vector of parameters with respect to the chosen copula, while  $g_k({\bf x}_{i}|\boldsymbol{\omega_k})$ represents the multivariate distribution constructed through the copula density $c_k$, explicitly:

	\[
	\begin{split}
	g_k({\bf x}_{i}|\boldsymbol{\omega}_k)&=c_k(F_1(x_{i,1}),\ldots,F_D(x_{i,D})|\boldsymbol{\omega}_k)\bigg(f_1(x_{i,1})\times \cdots \times f_D(x_{i,D})\bigg)\\
	&=c_k(F_1(x_{i,1}),\ldots,F_D(x_{i,D})|\boldsymbol{\omega}_k)\prod_{j=1}^{D}f_j(x_{i,j}).
	\end{split}
	\]
	Since there are different families of copulas, the parameter $\boldsymbol{\omega}_k$ is the one related to the specific copula that is chosen to model the cluster $k$. 
The goal of the mixture model is to find the optimal parameters in $\boldsymbol{\theta}$, in the equation \eqref{join_X_prob}, 
that maximize the log-likelihood $\mathcal{L}(\mathbf{X}|\boldsymbol\theta)$:
\begin{equation}\label{log_lik}
\underset{\boldsymbol\theta}{\operatorname{arg\,max}}\mathcal{L}(\mathbf{X}|\boldsymbol\theta)=\underset{\boldsymbol\theta}{\operatorname{arg\,max}}\log p(\mathbf{X}|\boldsymbol\theta)=\underset{\pi_k,\boldsymbol\omega_k}{\operatorname{arg\,max}}\sum_{i=1}^{n}\log\sum_{k=1}^{K}\pi_k g_k({\bf x}_{i}|\boldsymbol{\omega}_k).
\end{equation}
Usually, in order to solve the optimization problem in \eqref{log_lik}, 
the Expectation-Maximization (EM) \cite{Dempster1977} algorithm is employed together with mixture models. 
 To derive the probability that an observation $\mathbf{x}_{i}$, is drawn from $g_k$, we introduce the latent variable $\mathbf{z}_i=(z_{i1},\cdots,z_{iK})^T$ such that $z_{ik}\in\{0,1\}$ with $k\in\{1,\cdots,K\}$. In this context the introduction of latent variables is a common practice \cite{Bishop} that enhances both the theoretical framework and the rationale behind employing the expectation-maximization algorithm. We define the joint distribution $p({\bf x}_i, {\bf z}_i)$ in terms of a marginal distribution $p({\bf z}_i)$ and a conditional distribution $p({\bf x}_i|{\bf z}_i)$,
\begin{equation}\label{prod_rule}
	    	p({\bf x}_{i},\mathbf{z}_i):=p({\bf x}_{i}|\mathbf{z}_i)p(\mathbf{z}_i).
	\end{equation}
The marginal distribution for ${\bf z}_i$ is characterized by the mixing coefficients, with the specification that $p(z_{ik} = 1) = \pi_k$.

We know beforehand that each $z_{ik}$ occurs independently and that it can only take the value one when $k$ is equal to the cluster from which the observation comes from, then the overall probability is: 
\[
p(\mathbf{z}_i)=p(z_{i1}=1)^{z_{i1}}p(z_{i2}=1)^{z_{i2}}\cdots p(z_{iK}=1)^{z_{iK}}=\prod_{k=1}^{K}\pi_k^{z_{ik}},
\]
while the conditional distribution $p({\bf x}_i|{\bf z}_i)$ can be written as,
	\begin{equation}\label{pxz}
	p(\mathbf{x}_{i}|\mathbf{z}_i)=\prod_{k=1}^{K}g_k(\mathbf{x}_{i}|\boldsymbol{\bf\omega}_k)^{z_{ik}}.
	\end{equation}
	
 

	
Let us introduce ${\bf Z}$ as the matrix whose  $i^{th}$ row is the vector of the latent variables $\mathbf{z}_i$, then we have the overall joint distribution, 

\[
p(\mathbf{X},\mathbf{Z}|\boldsymbol{\theta})=\prod_{i=1}^{n}\prod_{k=1}^{K}\pi_k^{z_{i_k}} g_k({\bf x}_{i}|\boldsymbol{\omega}_k)^{z_{i_k}}
\]
and hence the log-likelihood is,
\begin{equation}\label{entireloglike}
\log (p(\mathbf{X},\mathbf{Z}|\boldsymbol{\theta}))=\sum_{i=1}^{n}\sum_{k=1}^{K}z_{i_k}\left(\log\pi_k+\log g_k({\bf x}_{i}|\boldsymbol{\omega}_k)\right).
\end{equation}







The formulated expression for the joint distribution $p(\mathbf{X},\mathbf{Z}|\boldsymbol{\theta})$ leads to significant simplifications in the EM algorithm. However this cannot be computed, since ${\bf z}_i$ is unknown, then we evaluate the
conditional probability of ${\bf z}_i$ given ${\bf x}_i$, and its value can be determined using Bayes' theorem,

\[
p(z_{ik}=1|{\bf x}_{i})=\frac{p({\bf x}_{i}|z_{ik}=1)\,p(z_{ik}=1)}{p(\bf{x}_i)}=
\frac{p({\bf x}_{i}|z_{ik}=1)\,p(z_{ik}=1)}{\displaystyle{\sum_{k'=1}^{K}\pi_{k'} p({\bf x}_{i}|z_{i{k'}}=1)p(z_{i{k'}}=1)}},
\]
where the quantity $p({\bf x}_i)$ has been computed via marginalization.
Knowing that $p(z_{ik}=1)=\pi_k$ and  $p({\bf x}_{i}|z_{ik}=1)=g_k({\bf x}_{i}|\boldsymbol{\omega}_k)$, then, the above equation becomes: 
\begin{equation}\label{responsability}
p(z_{ik}=1|{\bf x}_{i})=\frac{\pi_k g_k({\bf x}_{i}|\boldsymbol{\omega}_k)}{\displaystyle{\sum_{k'=1}^{K}\pi_{k'} g_{k'}({\bf x}_{i}|\boldsymbol{\omega}_{k'})}}.
\end{equation}
The quantity in \eqref{responsability} is called \textit{responsibility of $k$-th cluster to observation $i$} and from now on will be denoted with the symbol $\gamma_{ik}$. Then for every cluster we have an array of responsabilities. This quantity is crucial in the Expectation step of the EM algorithm  for the maximization of the complete log-likelihood.

\section{Expectation-Maximization for Copula Mixture Model }


We formalize the expectations maximization algorithm for copula mixture model in the general form, this can be implemented in different ways:
\begin{itemize}
	\item a completely parametric way in which the parameters of the copula and the parameters of the marginal probability densities are estimated. In this case there are two approaches: Inference For Marginal (IFM) \cite{Joe_1996}, and Expectation/Conditional Maximization ECM \cite{Meng_1993}. This last one, although it may work well for small-sized data,  requires high computational costs when there is a lack of a-priori knowledge about the marginals. In such cases, one needs to search for the distribution that fits the data well within a set of distributions;
	\item a semiparametric way 
in which, the estimation of marginals is approached empirically. 
The semiparametric nature of this approach strikes a balance between flexibility and computational efficiency.
\end{itemize}  
In this work, we adopt the semiparametric approach, leveraging the approximation properties of densities with the previously introduced BSHQI. Below, we describe the details of the proposed algorithm.

In the Expectation step, we employ the existing parameter values \( \boldsymbol{\theta}:=\boldsymbol{\theta}^{(t)} \) to determine the posterior distribution of the latent variables at the $t$-step of the algorithm, denoted as \( p({\bf Z}|{\bf X}, \boldsymbol{\theta}^{(t)}) \).  Subsequently, we utilize this posterior distribution to calculate the expectation of the complete-data log likelihood, evaluated for a general parameter vector \( \boldsymbol{\theta} \). This expectation, is called \textit{auxiliary function} represented as \( Q(\boldsymbol{\theta}, \boldsymbol{\theta}^{(t)}) \), is given by

\[
Q(\boldsymbol{\theta},\boldsymbol\theta^{(t)}) =
\mathbb{E}_{(\mathbf{Z}\mid\mathbf{X},\boldsymbol\theta^{(t)})}\left( \log p (\mathbf{X},\mathbf{Z}|\boldsymbol\theta)  \right).
\]
To simplify the notation, we shall consider $\mathbb{E}_{(\mathbf{Z}\mid\mathbf{X},\boldsymbol\theta^{(t)})}\left( \log p (\mathbf{X},\mathbf{Z}|\boldsymbol\theta)  \right)=\mathbb{E}\left(\log p(\mathbf{X},\mathbf{Z}|\boldsymbol\theta)\right)$ and so we have:
\begin{equation}\label{aux_func}
\begin{split}
Q(\boldsymbol\theta,\boldsymbol\theta^{(t)})&=\mathbb{E}[\log p(\mathbf{X},\mathbf{Z}|\boldsymbol\theta^{(t)})]=\sum_{\mathbf{Z}}p(\mathbf{Z}|\mathbf{X},\boldsymbol\theta^{(t)})\log p(\mathbf{X},\mathbf{Z}|\boldsymbol\theta)=\sum_{i=1}^{n}\sum_{k=1}^{K}\gamma_{ik}^{(t)}\left(\log \pi_k+\log g_k(x_{i}|\boldsymbol \omega_{k})\right)\\
&=\sum_{i=1}^{n}\sum_{k=1}^{K}\gamma_{ik}^{(t)}\bigg(\log \pi_k+\log \Big(c_k(\hat{F}_1(x_{i,1}),\ldots,\hat{F}_D(x_{i,D})|\boldsymbol \omega_{k}\Big)+\sum_{j=1}^{D} \log \hat{f}_{j}(x_{ij})\bigg)
\end{split}
\end{equation}
in which $\gamma_{ik}^{(t)}$ is the \textit{Responsibility} introduced in equation (\ref{responsability}).

\paragraph{Maximization step:} In the maximization step we update the parameters in $\boldsymbol\theta^{(t+1)}$ by computing:
\[
\boldsymbol\theta^{(t+1)} = \underset{\boldsymbol\theta}{\operatorname{arg\,max}} \ Q(\boldsymbol\theta,\boldsymbol\theta^{(t)})
\]
This is the most complex step of the algorithm and in the following we address separately the computation of the optimal parameters $\pi_k^{(t+1)}$ and $\boldsymbol{\omega}_k^{(t+1)}$.

For $\pi_k^{(t+1)}$, a closed form can be derived. Note that the maximization of the function $Q(\boldsymbol\theta\mid\boldsymbol\theta^{(t)})$ should take into account the restriction that $\sum_{k=1}^{K}\pi_k=1$. Hence, we can add a Lagrange multiplier to (\ref{aux_func}) ,

\[
Q(\boldsymbol\theta,\boldsymbol\theta^{(t)})=\sum_{i=1}^{n}\sum_{k=1}^{K}\gamma_{ik}^{(t)}\left(\log \pi_k+\log g_k(x_{i}|\boldsymbol \omega_k)\right)-\lambda \Bigg(\sum_{k=1}^{K}\pi_k-1\Bigg).
\]
Taking the derivative of $Q$ with respect to $\pi_k$ and setting it equal to zero, leads to
\begin{equation}\label{der_pi}
    \frac{\partial Q(\boldsymbol\theta,\boldsymbol\theta^{(t)})}{\partial \pi_k}=\sum_{i=1}^{n}\frac{\gamma_{ik}^{(t)}}{\pi_k}-\lambda=0.
\end{equation}

Then, by rearranging the terms and applying a summation over $k$ to both sides of the equation, we obtain:
\[
\sum_{i=1}^{n}\gamma_{ik}^{(t)}=\pi_k\lambda \rightarrow \sum_{k=1}^{K}\sum_{i=1}^{n}\gamma_{ik}^{(t)}=\sum_{k=1}^{K}\pi_k\lambda
\]
we know that the summation of all mixing coefficients equals one. In addition, we know that summing up the responsabilities $\gamma$ over $k$ will also give us 1. Thus we get $\lambda = n$. Using this result, we can solve \eqref{der_pi} for $\pi_k$:
\[
\pi_k^{(t+1)}:=\frac{\sum_{i=1}^{n}\gamma_{ik}^{(t)}}{n}.
\]

Regarding the optimization with respect to $\boldsymbol{\omega}_k$, systematic decomposition into distinct maximization steps is carried out in order to get the 
 optimization of the complete data log-likelihood with respect to the parameters of the model. 
 


One notable feature of our algorithm lies in its inherent flexibility in the selection of copulas from the outset. When a single copula is chosen initially, the subsequent maximization step, following the updating of marginals, exclusively targets the parameters of the selected copula until the log-likelihood function converges. Conversely, when the initial choice encompasses various copulas, a comprehensive fitting process is initiated after updating the marginals. The copula that best aligns with the updated data is then chosen based on the maximum likelihood, enhancing the adaptability and performance of our algorithm in diverse scenarios, formally:
\begin{itemize}
	\item \textbf{Maximization first step:} For each cluster, use the data ${\bf X}_{j}$, $j \in \{1,2,\ldots,D\}$ to update the CDFs $\hat{F}_{j}$ according to the equation (\ref{eq:weightedCDF}) in which the weights are the \emph{responsabilities} and compute the PDFs $\hat{f}_{j}$, $j \in \{1,2,\ldots,D\}$  with the BSHQI strategy described in Section \ref{sub:qispline}.
	\item \textbf{Maximization second step (one copula):} By looking at \eqref{aux_func}, for each clusters, we need to maximize only the following,
\begin{equation}\label{eq:om_max}
	\sum_{i=1}^{n}\sum_{k=1}^{K}\gamma_{ik}^{(t)}\bigg(\log \Big(c_k(\hat{F}_1(x_{i,1}), \ldots, \hat{F}_D(x_{i,D});\boldsymbol \omega_k\Big)\bigg),
	\end{equation}
	with respect to $\boldsymbol \omega_k$ for the optimal copula parameters. 
		\item \textbf{Maximization second step (two or more copulas):} If the choice of different copulas is enabled, then, the optimization of \eqref{eq:om_max} is carried out also for each copula. 

 Note that, in this case for the maximization of the log-likelihood we use the  Limited-memory Broyden, Fletcher, Goldfarb, Shanno (L-BFGS-B) method\footnote{\url{https://docs.scipy.org/doc/scipy/reference/optimize.minimize-lbfgsb.html}}, that is an optimization algorithm in the family of quasi-Newton methods that approximates the Broyden–Fletcher–Goldfarb–Shanno algorithm (BFGS) using a limited amount of computer memory \cite{galen2007, malouf2002}.
\end{itemize}

To start the algorithm we use the following initialization procedure:



\paragraph{Initialization} 

\begin{itemize}
\item[0.] Choose a set $S$ of copulas function among Gaussian, Clayton, Gumbel, Frank.
    \item[1.] Create a random clustering partition of the given data. We call $X_{j,r}$ the subset of the $j$-th column given by the produced partition.
    \item[2.] For every cluster, compute the the marginals $\hat F_j(X_{j,r})$ for  $j=1,\ldots,D$ either with the KDE methodology or with the BSHQI (the choice is at the user discretion). 
    \item[3.] For every cluster, find the best copula in $S$ according to the maximum likelihood achieved and keep their parameters $\boldsymbol{\omega}_k$ fitted by the produced data $u_1=\hat F_1(X_{1,r}),u_2=\hat F_2(X_{2,r}), \ldots, u_D=\hat F_D(X_{D,r})$, 
    and compute the associated $\pi_k$ for $k=1,\ldots,K$.
    \item[4.] Repeat steps 1-3 for $5$ times and select the copulas partition with the maximum likelihood  so that 
    $\boldsymbol{\theta}^{(0)} := \{\pi_k^{(0)}, \boldsymbol{\omega}_k^{(0)}\}= \{\pi_k^{\mbox{best}}, \boldsymbol{\omega}_k^{\mbox{best}} \}$ for $k=1,\ldots, K$.
\end{itemize}

 


\paragraph{Expectation} 
In this step we evaluate the responsabilities. 
\paragraph{Maximization} 
The algorithm as just described monotonically approaches a local minimum of the cost function.
After computing the new estimates, we set $\theta^{(t)} =(\pi_k^{(t)},\omega_k^{(t)})$ for $k = 1\ldots,k$, and go to the next Expectation step.
We set the tolerance $tol\leq 10^{-4}$ than, the best parameters are obtained when the the convergence of the log-likelihood is reached, i.e.
\begin{equation}
    \frac{|\mathcal{L}(\mathbf{X}|\boldsymbol\theta^{(t+1)})-\mathcal{L}(\mathbf{X}|\boldsymbol\theta^{(t)})|}{1+|\mathcal{L}(\mathbf{X}|\boldsymbol\theta^{(t+1)})|}<tol \quad \textnormal{for} \quad t=1, \ldots, (\emph{iter}-1).
\end{equation}

\section{Experiments}\label{change_experiments}

In this section, we describe the experiments conducted to validate the accuracy of our algorithm. Specifically, we commence by examining synthetic datasets constructed with diverse families of copulas. Subsequently, we delve into the assessment of real-world datasets.  The clustering strategy using copulas and density estimation with the proposed BSHQI technique, called CopMixM\_BSHQI, will be compared by using marginals estimation with kernel density estimation using the Python package KDEpy, referred to as CopMixM\_KDEpy here, which is found to be the fastest among various python packages for empirical density estimators as it utilizes the convolution Fast Fourier Transform for calculations. Additionally, we will compare the results obtained with clustering algorithms available from \texttt{scikit-learn}: Gaussian Mixture Model (GMM), where we set the initialization to `\texttt{random}', the tolerance equal to $10^{-4}$ (as in our algorithm) and default parameters setting otherwise, and K-Means  where we set the initialization to \texttt{random}. To measure the achieved performance we use some classical metrics, available from \texttt{scikit-learn}, such as:
\begin{itemize}
    \item Silhouette Score: a value between $-1 $ and $1$, with $1$ being the best;
    \item Calinski-Harabasz Index:   highest values indicate better performance.
    \item  Davies-Bouldin Score: the minimum is $0$, lower values indicate better output.
\end{itemize}
In those tests where ground truth labels are accessible, the evaluation  is done also by using the Adjusted Rand Score, Homogeneity Score, Rand Score, and Completeness Score which are permutation invariant and the highest value $1$ indicates optimal clustering.

\subsection{Synthetic Dataset}

\begin{figure}[htbp]
		\centering
		\subfigure[$\mathcal{X}_1$]{\includegraphics[scale=0.2]{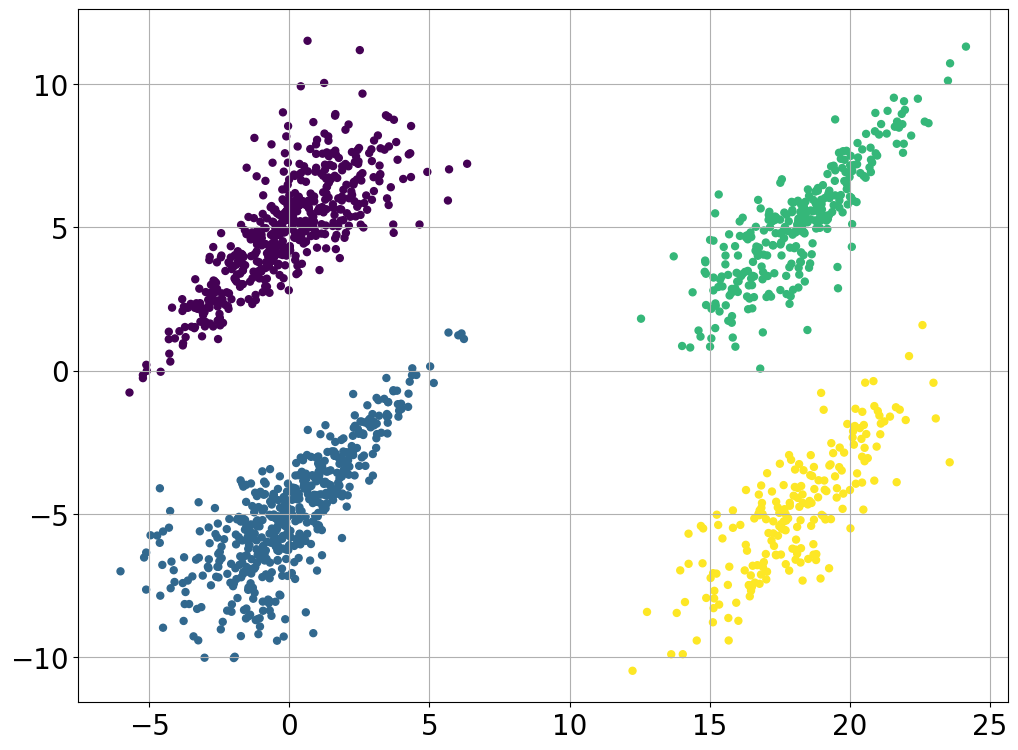}}\quad
		\subfigure[$\mathcal{X}_2$]{\includegraphics[scale=0.2]{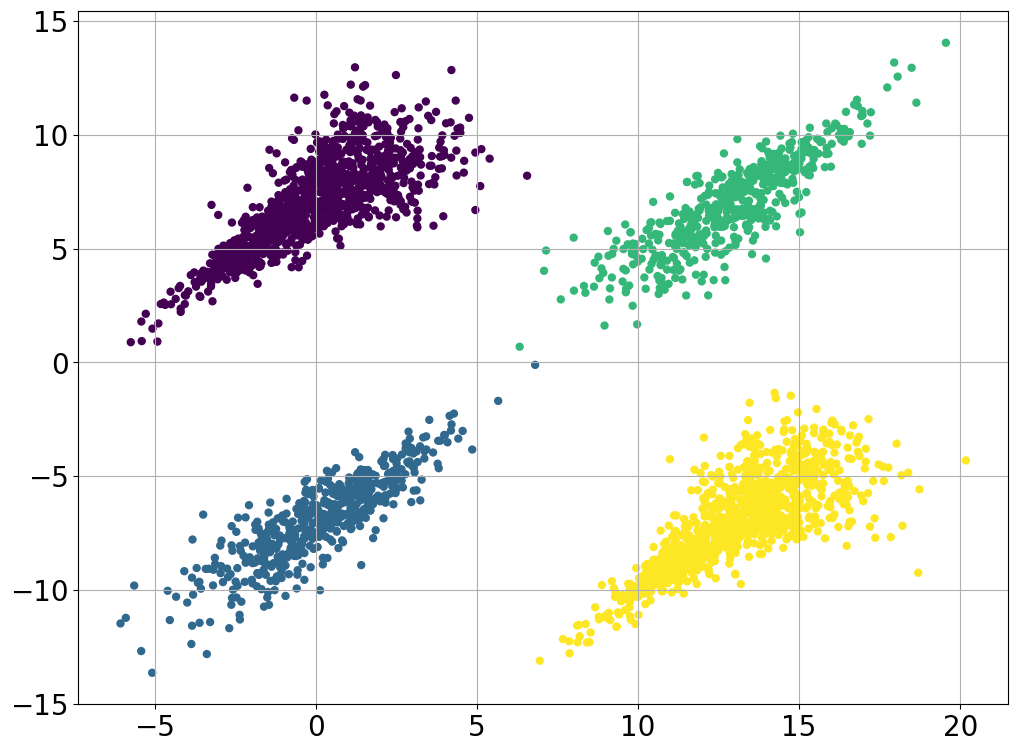}}\quad
  \subfigure[$\mathcal{X}_3$]{\includegraphics[scale=0.2]{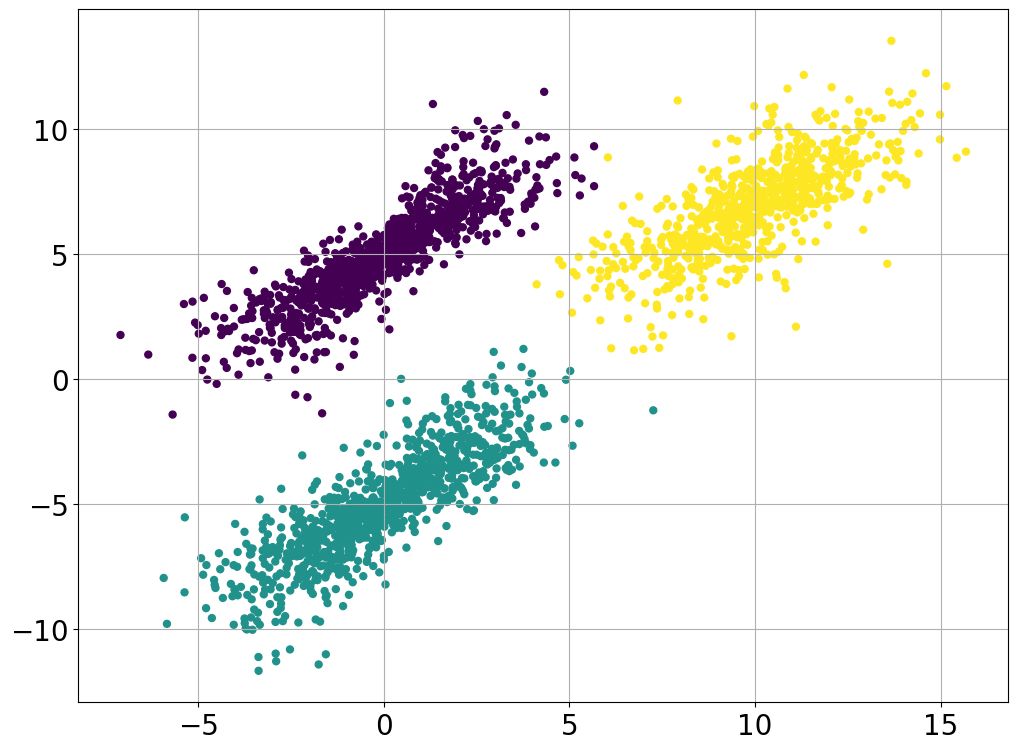}}\\
      \subfigure[$\mathcal{X}_4$]{\includegraphics[scale=0.3]{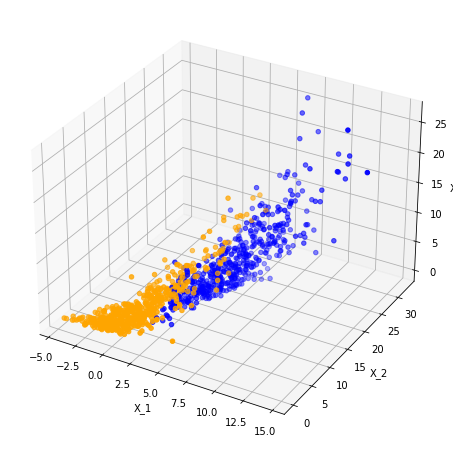}}\quad
    \subfigure[Pairwise $\mathcal{X}_4$]{\includegraphics[scale=0.25]{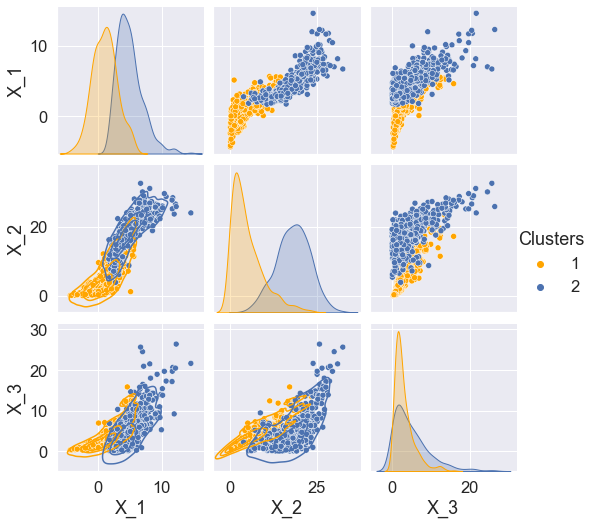}}\quad
		\caption{Synthetic dataset: (a) Ground truth $\mathcal{X}_1$, (b) Ground truth $\mathcal{X}_2$, (c) Ground truth $\mathcal{X}_3$, (d) Ground truth $\mathcal{X}_4$, (e) Pairwise ground truth $\mathcal{X}_4$}
		\label{fig:gt_synt}
\end{figure}
	
\begin{figure}[htbp]
		\centering
		\subfigure[GMM]{\includegraphics[scale=0.2]{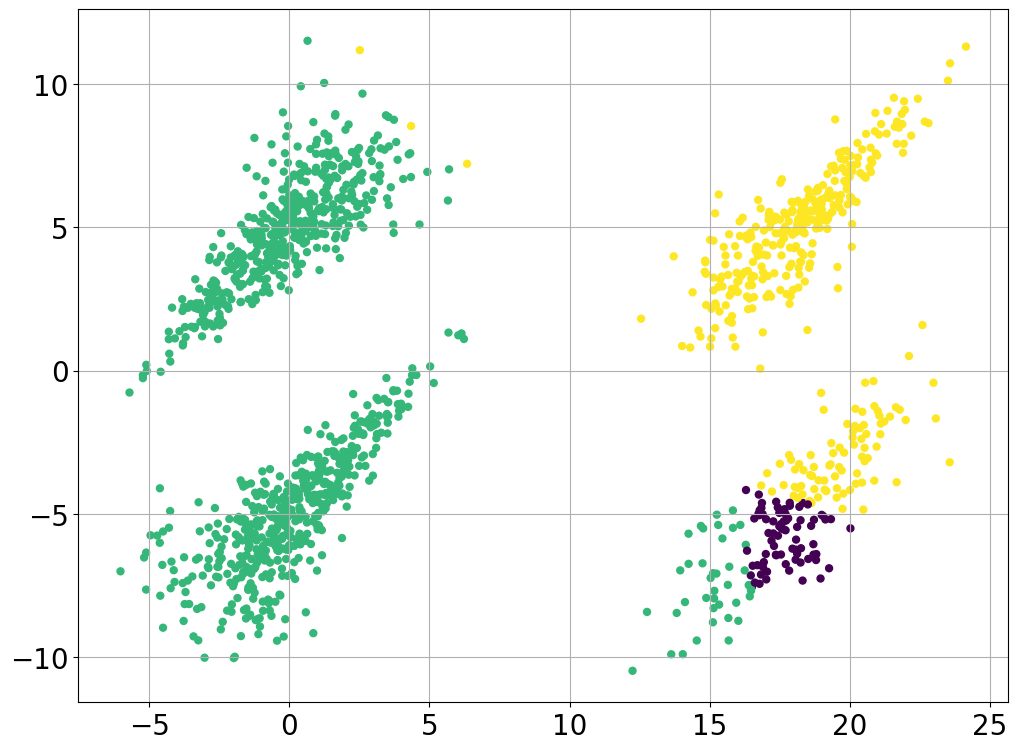}}\quad
		\subfigure[CopMixM\_BSHQI]{\includegraphics[scale=0.2]{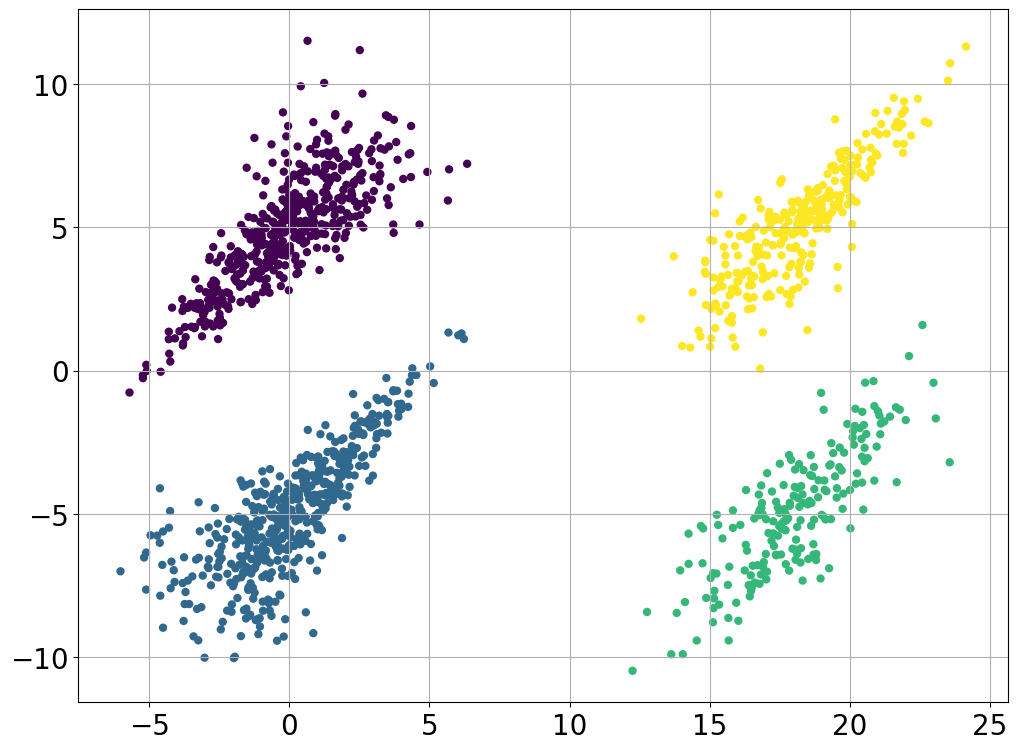}}\quad
  \subfigure[CopMixM\_KDEpy]{\includegraphics[scale=0.2]{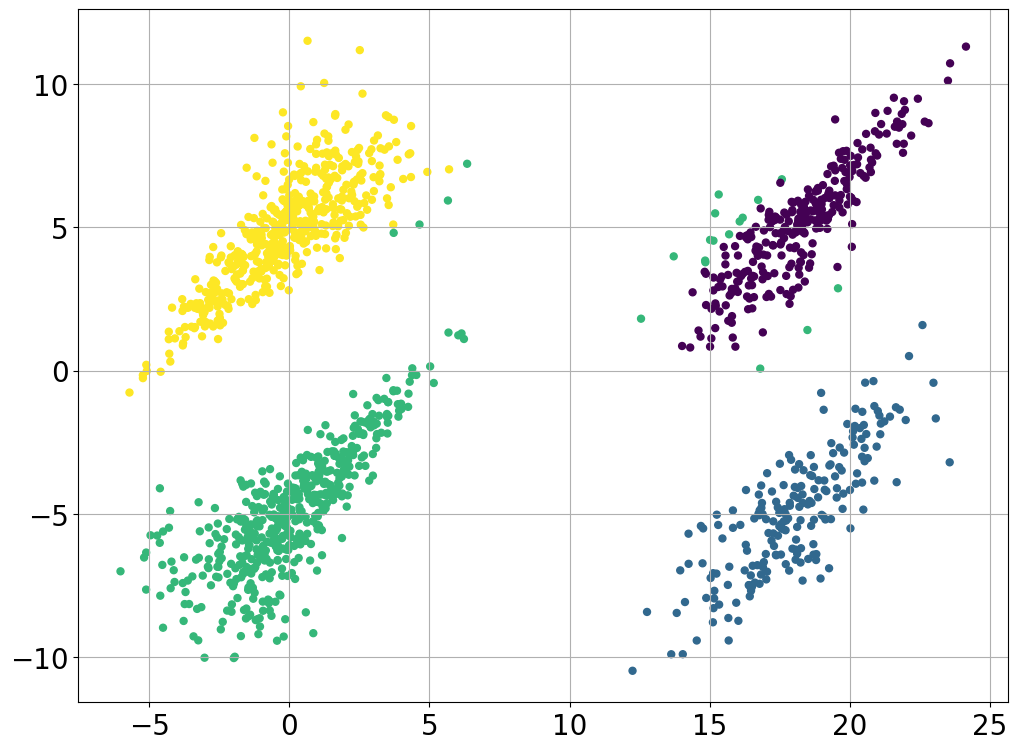}}\quad
		\caption{Synthetic dataset $\mathcal{X}_1$: (a) GMM, (b) CopMixM\_BSHQI, (c) CopMixM\_KDEpy}
		\label{fig:synt1_results}
  \end{figure}

\begin{table}
\centering
 
\begin{adjustbox}{max width=1\textwidth}
\begin{tabular}{lcccc}
\hline
\hline
&K-Means& GMM & CopMixM\_BSHQI & CopMixM\_KDEpy \\
\hline \hline
 Silhouette Score & 0.660      &   0.512 &        {\bf 0.659} &  0.628             \\
 Calinski-Harabasz Index& 5977  & 1570     &     {\bf 5890}     &  3940             \\
 Davies-Bouldin Score& 0.482   &    0.529 &        {\bf 0.486} &  0521             \\
  \hline            
 Adjusted Rand Score&  0.974   &    0.437 &        {\bf 1}     &  0.963             \\
 Homogeneity Score& 0.963     &    0.455 &        {\bf 1}     &  0.952             \\
 Rand Score&    0.989          &    0.707 &        {\bf 1}     & 0.985              \\
 Completeness Score&  0.964   &    0.797 &        {\bf 1}     & 0.957              \\
\hline
\hline
\end{tabular}
\end{adjustbox}
\caption{Clustering metrics for Synthetic Dataset $\mathcal{X}_1$. In bold the best values.}
\label{tab:synt1_results}
\end{table}

\begin{table}
	\centering
    
	\begin{adjustbox}{max width=1\textwidth}
	\begin{tabular}{c||c|c||c|c}
		\hline
  \hline
		\bf Clusters & \# Points & Different Copulas  & \# Points& One Copula  \\ 
		\hline
  \hline
		&   &  & &\\
		\bf 1           & 500  & Gaussian   & 505 &Gaussian\\
  \bf 2         & 500 & Gaussian   & 498 & Gaussian\\
  \bf 3            & 300   & Gaussian  & 299 &Gaussian\\
		\bf    4        & 200   & Clayton   &198 &Gaussian \\
  &  &    & &\\
  \hline
	Log-likelihood & \multicolumn{2}{c||}{\bf -7535} & \multicolumn{2}{c}{-8125}\\
            
		\hline

	\end{tabular}

	\end{adjustbox}
    \caption{Results of CopMixM\_BSHQI for the Synthetic Dataset $\mathcal{X}_1$ obtained with different copulas and the Gaussian copula.}
    \label{tab:synt1_diffcop}
\end{table}

\begin{figure}
		\centering
		\subfigure[GMM]{\includegraphics[scale=0.2]{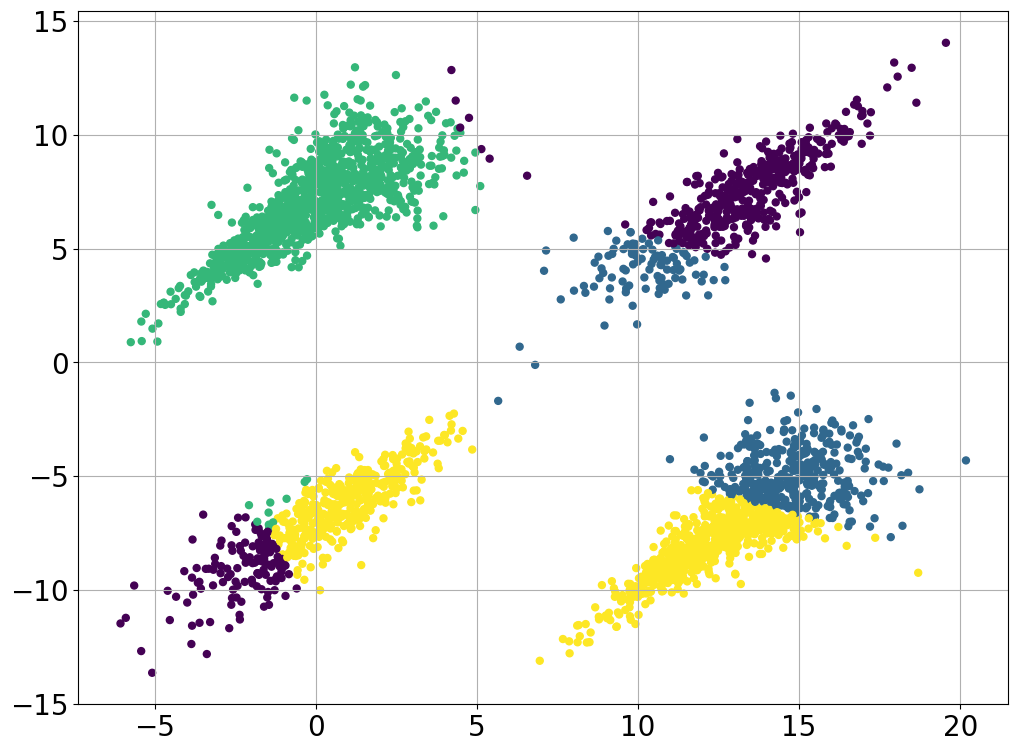}}\quad
		\subfigure[CopMixM\_BSQI]{\includegraphics[scale=0.2]{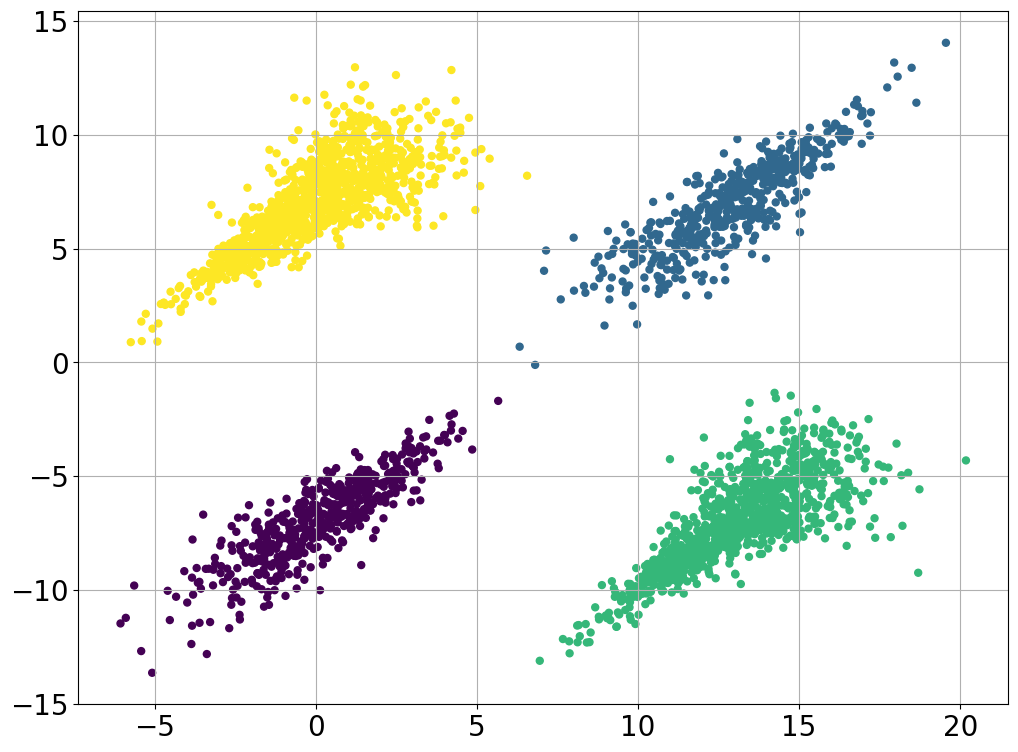}}\quad
  \subfigure[CopMixM\_KDEpy]{\includegraphics[scale=0.2]{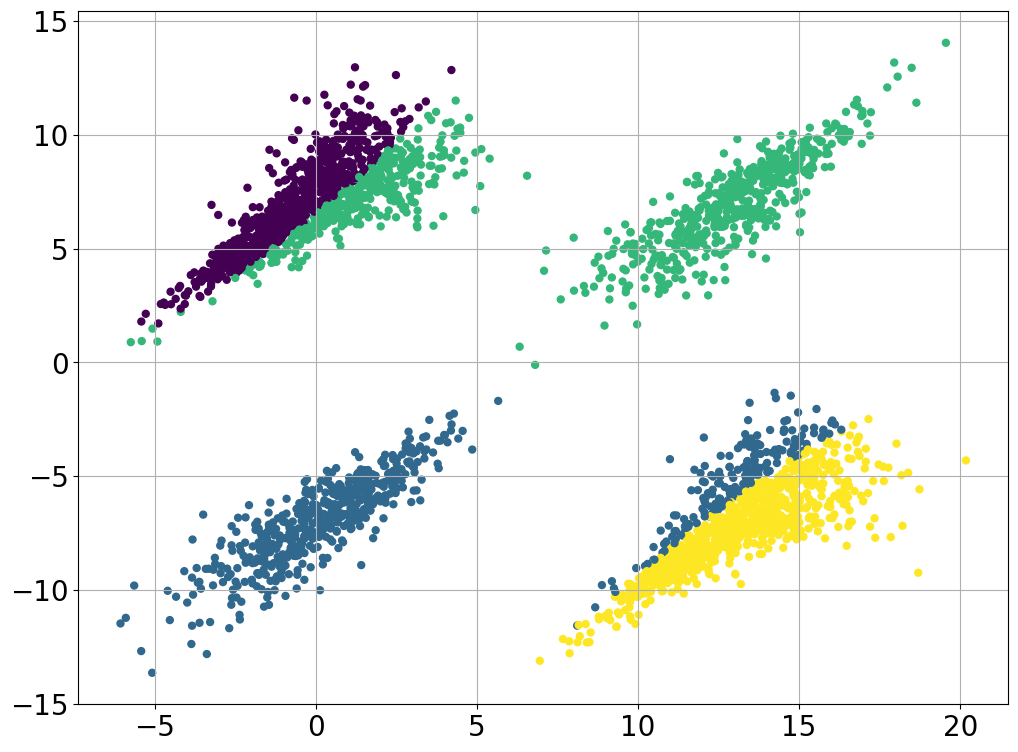}}\quad
		\caption{Synthetic dataset $\mathcal{X}_2$: (a) GMM, (b) CopMixM\_BSQI, (c) CopMixM\_KDEpy}
		\label{fig:synt2_results}
\end{figure}

\begin{table}
\centering
\begin{adjustbox}{max width=1\textwidth}
\begin{tabular}{lcccc}
\hline
\hline
& K-Means& GMM & CopMixM\_BSHQI & CopMixM\_KDEpy \\
\hline \hline
 Silhouette Score& {\bf 0.726}       &    0.269 &        {\bf  0.726} &     0.370  \\
 Calinski-Harabasz Index&11495  & 2640     &    {\bf 11500}     &   2990          \\
 Davies-Bouldin Score& {\bf 0.365}    &    1.48  &         {\bf 0.365}  &      0.964         \\
 \hline
 Adjusted Rand Score& 0.998    &    0.382 &       {\bf 1} &        0.605      \\
 Homogeneity Score&0.994       &    0.488 &        {\bf 0.999} &       0.738        \\
 Rand Score&   0.999           &    0.758 &        {\bf 1} &       0.845       \\
 Completeness Score& 0.994     &    0.512 &        {\bf 0.998} &          0.720     \\
\hline
\hline
\end{tabular}
\end{adjustbox}
\caption{Clustering metrics for Synthetic Dataset $\mathcal{X}_2$. In bold the best values.}
\label{tab:synt2_diffcop}
\end{table}
\begin{table}
	\centering
    
	\begin{adjustbox}{max width=1\textwidth}
	\begin{tabular}{l|c|c||c|c}
		\hline
  \hline
		\bf Clusters & \# Points & Different Copulas  & \# Points & One\_Copula \\ 
		\hline
  \hline
		\textbf{1} & 1013 & Clayton & 1152 &Gaussian \\
		\textbf{2} & 1014 & Clayton & 980 & Gaussian\\
		\textbf{3} & 985 & Gumbel  & 858 & Gaussian\\
		\textbf{4} & 988 & Gumbel   & 1010 & Gaussian\\
   \hline
	Log-likelihood & \multicolumn{2}{c||}{\bf -13324} & \multicolumn{2}{c}{-14873}\\
            
		\hline
		\hline

	\end{tabular}
	\end{adjustbox}
    \caption{Results of CopMixM\_BSHQI for the Synthetic Dataset $\mathcal{X}_2$ obtained with different copulas and the Gaussian copula.}
    \label{tab:synt2_results}
\end{table}

\begin{figure}
		\centering
		\subfigure[GMM]{\includegraphics[scale=0.2]{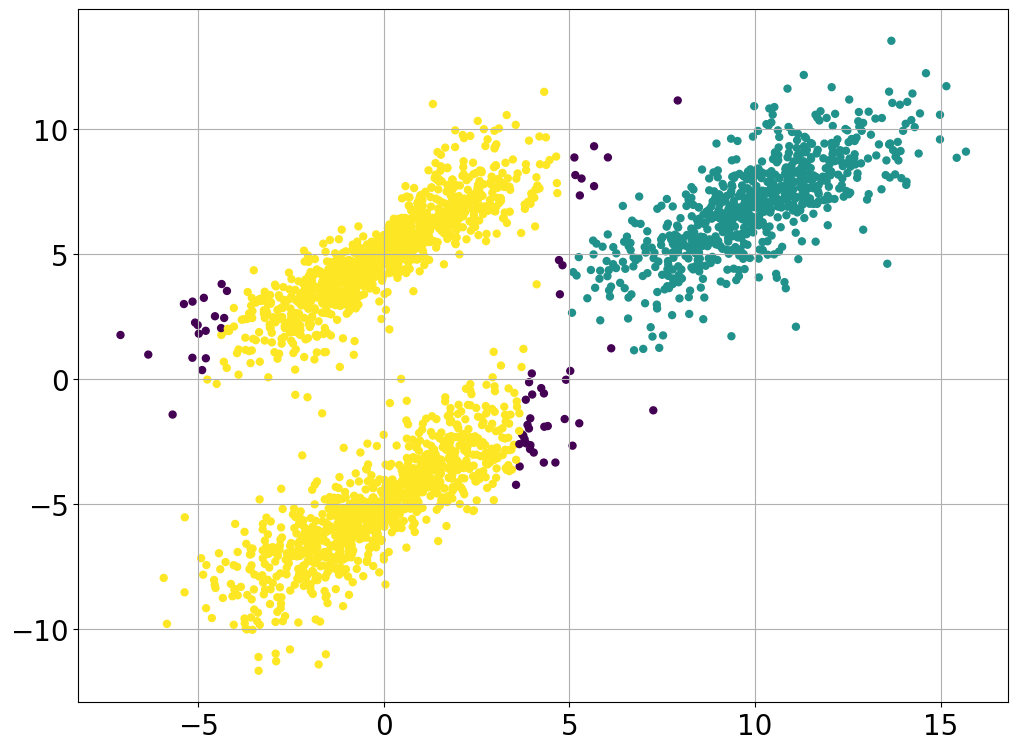}}\quad
		\subfigure[CopMixM\_BSQI]{\includegraphics[scale=0.2]{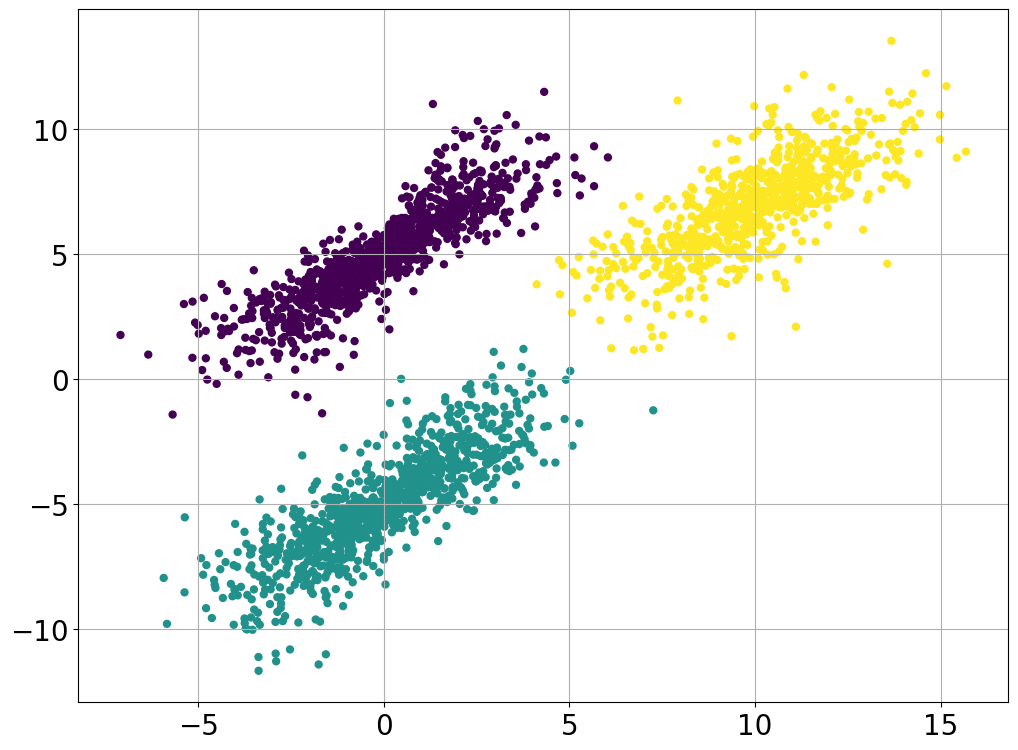}}\quad
  \subfigure[CopMixM\_Emp]{\includegraphics[scale=0.2]{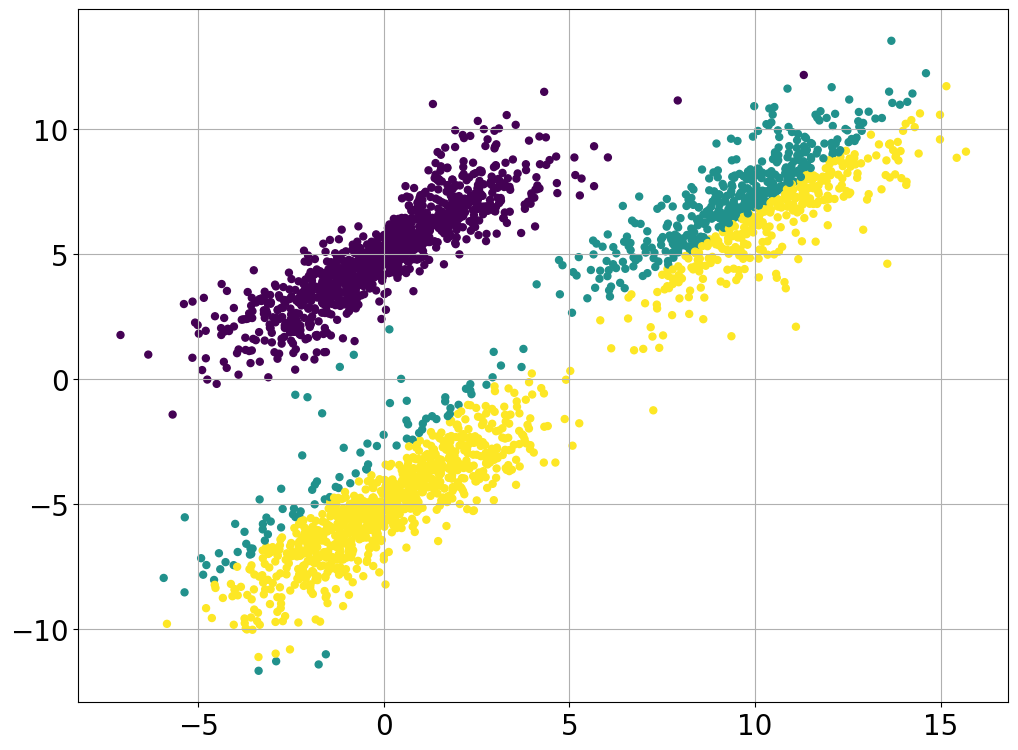}}\quad
		\caption{Synthetic dataset $\mathcal{X}_3$: (a) GMM, (b) CopMixM\_BSQI, (c) CopMixM\_KDEpy}
		\label{fig:synt3_results}
\end{figure}

\begin{table}
\centering
 
\begin{adjustbox}{max width=1\textwidth}
\begin{tabular}{lcccc}
\hline
\hline
& K-Means & GMM & CopMixM\_BSHQI & CopMixM\_KDEpy \\
\hline \hline
 Silhouette Score &{\bf 0.638} &    0.253 &        {\bf 0.638} &          0.384     \\
 Calinski-Harabasz Index &{\bf 7874}  & 1420     &      7710     &             702  \\
 Davies-Bouldin Score    &{\bf 0.474}  &    3.8   &       0.476&          1.47    \\
 \hline
 Adjusted Rand Score     &0.960 &    0.476 &        {\bf 0.999} &      0.615       \\
 Homogeneity Score       & 0.932&    0.518 &        {\bf 0.997} &          0.598    \\
 Rand Score              &0.982 &    0.724 &        {\bf 1}     &           0.804   \\
 Completeness Score      &0.931 &    0.841 &        {\bf 0.997} &            0.771   \\
\hline
\hline
\end{tabular}
\end{adjustbox}
\caption{Clustering metrics for Synthetic Dataset $\mathcal{X}_3$. In bold the best values.}
\label{tab:synt3_results}
\end{table}
\begin{table}
	\centering
    
	\begin{adjustbox}{max width=1\textwidth}
	\begin{tabular}{c|c|c||c|c}
		\hline
  \hline
		\bf Clusters & \# Points & Different Copulas  & \# Points& One Copula  \\ 
		\hline
  \hline
		\textbf{1} & 702 & Frank  &  698 &Gaussian \\
		\textbf{2} & 1001 & Frank  & 935 & Gaussian\\
		\textbf{3} & 998 & Frank  & 1067 & Gaussian\\
	
    \hline
	Log-likelihood & \multicolumn{2}{c||}{\bf -12644} & \multicolumn{2}{c}{-13835}\\
		\hline

	\end{tabular}
	\end{adjustbox}
    \caption{Results of CopMixM\_BSHQI for the Synthetic Dataset $\mathcal{X}_3$ obtained with different copulas and the Gaussian copula.}
    \label{tab:synt3_diffcop}
\end{table}

\begin{figure}
	\centering
    \subfigure[GMM]{\includegraphics[scale=0.25]{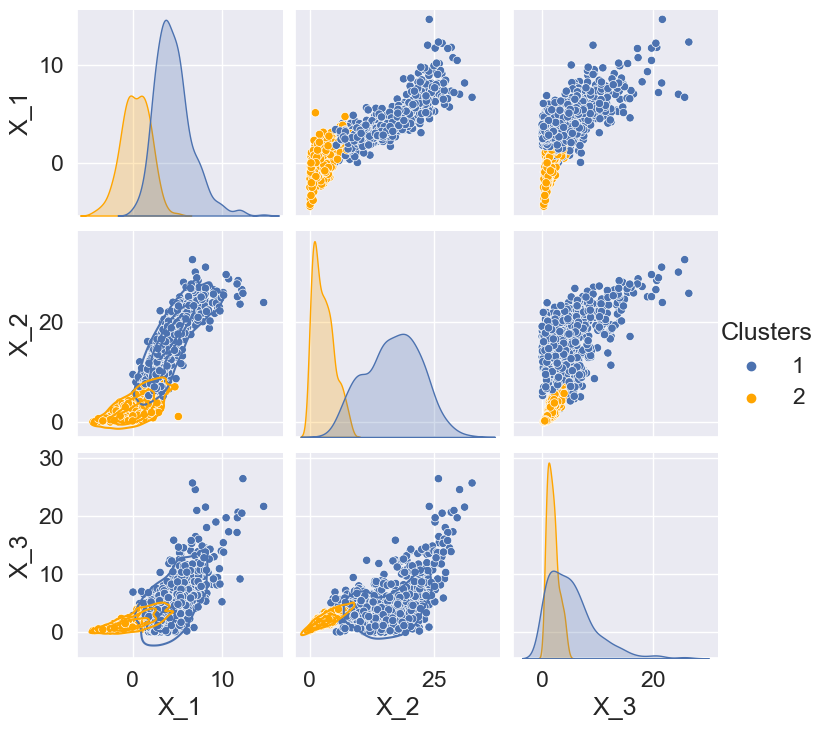}}\quad
    \subfigure[CopMixM\_BSHQI]{\includegraphics[scale=0.25]{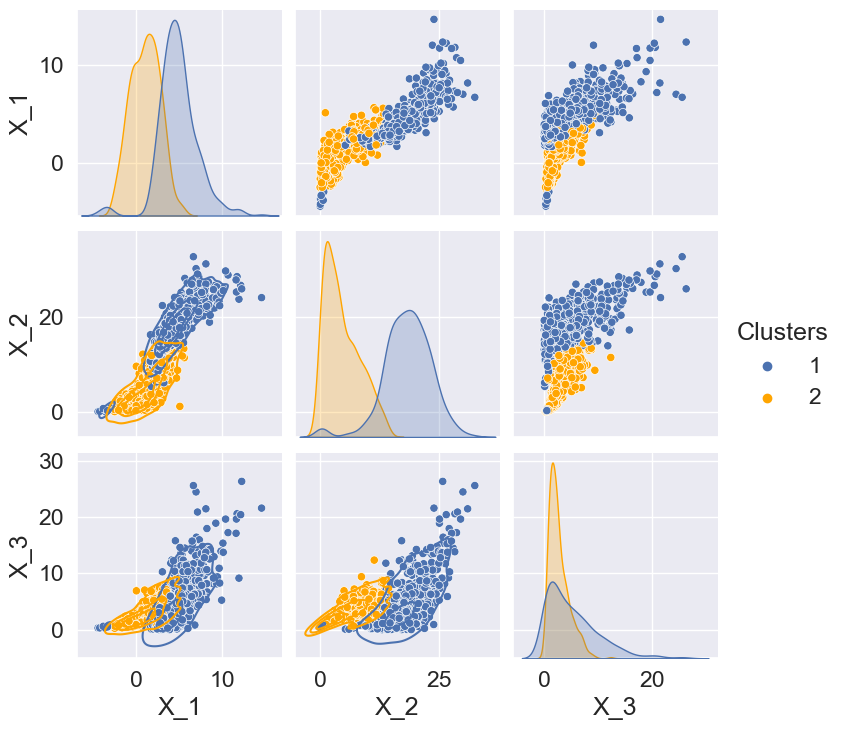}}\quad
	\caption{Synthetic dataset $\mathcal{X}_4$: (a) GMM , (b) CopMixM\_BSHQI}
	\label{fig:synt_3d}
\end{figure}
\begin{table}
\centering
 \begin{adjustbox}{max width=1\textwidth}
\begin{tabular}{lcccc}
\hline
\hline
& K-Means& GMM & CopMixM\_BSHQI & CopMixM\_KDEpy \\
\hline
 Silhouette Score&{\bf 0.576} & 0.504 &         0.561 &    0.536           \\
 Calinski-Harabasz Score&{\bf 2030} & 1420     &     1850     &     1630         \\
 Davies-Bouldin Score& 0.601   &    {\bf 0.592} &        0.618 &        0.661     \\
 \hline
 Adjusted Rand Score& 0.595    &    0.544 &        {\bf 0.695} &       0.685         \\
 Homogeneity Score&  0.488     &    0.528 &        0.589 &        {\bf 0.659}        \\
 Rand Score&   0.797           &    0.772 &        {\bf 0.848} &        0.843       \\
 Completeness Score&  0.489    &    0.555 &        0.594 &         {\bf 0.595}      \\
\hline
Miss-classification Rate: & 0.12 &0.11 ($\pm 0.02$) & {\bf 0.08} ($\mathbf{\pm 0.01}$) & 0.09 ($\pm 0.01$)\\
\hline
\end{tabular}
\end{adjustbox}
\caption{Clustering metrics for Synthetic Dataset $\mathcal{X}_4$. In bold the best values.}
\label{tab:synt4_results}
\end{table}


We examine four synthetic datasets $\mathcal{X}_1, \mathcal{X}_2,\mathcal{X}_3, \mathcal{X}_4$ characterized as follows: 
\begin{itemize}
\item the samples in $\mathcal{X}_1$ are drawn from two Claytons copula, one Frank and one Gumbel, see Figure \ref{fig:gt_synt} (a);
\item the samples in $\mathcal{X}_2$ constitutes $4$ clusters and  drawn from  2 Clayton copulas and 2 Gumbel copulas, see Figure \ref{fig:gt_synt} (b);
\item the samples in $\mathcal{X}_3$ constitutes $3$ clusters drawn from the Frank copula, see Figure \ref{fig:gt_synt} (c);
\item the samples in $\mathcal{X}_4$ constitutes $2$ clusters of 3D scattered points, see Figure \ref{fig:gt_synt} (d)-(e). 
\end{itemize}

 All the presented results are obtained with random initialization and with the selection of different copulas. Regarding dataset $\mathcal{X}_1$, in Figure \ref{fig:synt1_results}, the outcomes from  GMM, CopMixM\_BSHQI and clustering with copulas fitting the marginals using the KDEpy approach are shown. Visually, the best results are obtained with our approach, see Figure \ref{fig:synt1_results} (b), and KDEpy Figure \ref{fig:synt1_results} (c), as they closely aligns with the ground truth in Figure \ref{fig:gt_synt} (a). The metrics detailed in Table \ref{tab:synt1_results} show a more quantitative comparison for the four used clustering algorithms. 
Notably, CopMixM\_BSHQI outperforms  GMM and gives better results compared to the other  methods with respect to various metrics. The highest Silhouette Score, Calinski-Harabasz Index, and Homogeneity Score indicate a superior cluster quality and better separation. For the Davies-Bouldin Score, where a lower value is desirable in presence of well-defined and compact clusters, CopMixM\_BSHQI again excels.
The Adjusted Rand Score, Rand Score, and Completeness Score further support the dominance of CopMixM\_BSHQI, as it achieves perfect scores$=1.0$ in these metrics, signifying strong agreement with the ground truth. In contrast, K-Means, GMM and CopMixM\_KDEpy exhibit lower scores, reflecting a lower level of agreement with the true cluster assignments.
Moreover, in Table \ref{tab:synt1_diffcop}, we present the results of CopMixM\_BSHQI on the synthetic dataset $\mathcal{X}_1$ under various copula configurations. Two primary scenarios were considered: one utilizing diverse copulas for specific clusters, and the other employing a single Gaussian copula, which is commonly used in practice and is often considered the default choice in mixture models.
 Evaluation was based on the maximum likelihood for the choice of the cluster's copula and the Log-likelihood for the overall mixture model performance. The results suggest to adopt diverse copulas tailored to specific clusters rather than using a single one. Indeed, the higher log-likelihood highlights the superior model adaptability, to the considered synthetic dataset $\mathcal{X}_1$, compared to the simplistic use of a single Gaussian copula.

Moreover, with respect to datasets $\mathcal{X}_2$ and $\mathcal{X}_3$, the findings depicted in Figures \ref{fig:synt2_results} and \ref{fig:synt3_results}, along with the corresponding metrics presented in Tables \ref{tab:synt2_results} and \ref{tab:synt3_results}, confirm the superiority of the CopMixM\_BSHQI methodology over the other three methodologies. This validate that CopMixM\_BSHQI works better not only in effectively distinguishing distinct clusters but also in appropriately assigning the suitable copula for each of them. This is further illustrated in Tables \ref{tab:synt2_diffcop} and \ref{tab:synt3_diffcop}, which show again the better performance in the use of diverse copulas with respect to a single copula. In these instances as well, the  log-likelihood supports this conclusion.

In the latest experiment, we show the effectiveness of the algorithm using the synthetic dataset \(\mathcal{X}_4\), which is in 3D  and comprises of two clusters. The results are presented in Figure \ref{fig:synt_3d} where only GMM and CopMixM\_BSHQI are shown as the results obtained with K-Means and CopMixM\_KDEpy are similar to the last one. Moreover in Table \ref{tab:synt4_results} the achieved metrics are reported. It can be seen that the worst performance is given by GMM, while the best sometimes is achieved by CopMixM\_KDEpy, altought the discrepancy with CopMixM\_BSHQI is minimal. For this test, since there are only two classes, we have also included accuracy in terms of 
miss-classification rate. As can be observed in this scenario  the CopMixM\_BSHQI performance surpasses that of the conventional K-Means, GMM and KDEpy based approach.
In summary, these results highlight the algorithm's effectiveness in capturing complex structures within synthetic datasets, positioning it as a robust choice for clustering activities in similar contexts. The findings suggest that the choice of copula significantly influences goodness of fit, with specific copula types being favored by certain clusters. Therefore, the model has an overall robust performance, particularly in contexts where the option to choose from multiple copulas is available, compared to relying on a single copula.

\subsection{Real Datasets}

We conduct experiments to validate the effectiveness of the CopMixM\_BSHQI algorithm on real-world datasets. Our analysis considers three datasets: the Australian Institute of Sport (AIS), Breast Cancer Wisconsin, and a case study involving the clustering of textual data using the 20newsgroups dataset\footnote{\url{https://scikit-learn.org/0.19/datasets/twenty_newsgroups.html}}. For the AIS\footnote{\url{https://rdrr.io/cran/GLMsData/man/AIS.html}} and Breast Cancer Wisconsin datasets\footnote{\url{https://archive.ics.uci.edu/dataset/14/breast+cancer}}, we compare our approach against a specific method referenced in \cite{Vine_SAHIN2022}, where our initialization settings are considered the same as in \cite{Vine_SAHIN2022} wherever possible. In the case of the 20newsgroups dataset, we compare our results with the conventional GMM.
\subsubsection{AIS}

The AIS dataset, comprising 13 measurements collected from 102 male athletes and 100 female athletes. The primary objective of clustering this dataset is to assess the effectiveness of our methodologies in producing results aligned with the given ground-truth. We consider the same settings as in \cite{Vine_SAHIN2022}, focusing on a subset of five variables: lean body mass (LBM), weight (Wt), body mass index (BMI), white blood cell count (WBC), and percentage of body fat (PBF).

This particular example exhibits a non-Gaussian distribution, as the contours are non-elliptical and displays asymmetric dependency patterns. Therefore, the application of copulas seems to be the optimal one for capturing this more complex form of dependency. Our algorithm was ran with both K-Means and random initialization. Furthermore, we explored the use of different copula choices. We utilized the miss-classification rate metric to assess the performance of the clustering algorithm. The results obtained with our approach, employing both types of initialization, are presented in the Table \ref{tab:ais_results} alongside with those reported in \cite{Vine_SAHIN2022}.

\begin{table}[htbp]
\centering
\begin{adjustbox}{max width=1\textwidth}
\begin{tabular}{lcccc}
\hline
\hline
& k-means & \makecell{VCMM\\k-means}   &\makecell{CopMixM\_BSHQI\\k-means} &\makecell{CopMixM\_BSHQI\\random} \\
\hline 
& & & \\
Miss-classification rate (Bins=$\lceil n^{1/3}\rceil$)  & 0.21 & {\bf 0.040}  &   0.090    & 0.045  \\
Miss-classification rate (Bins=Rice)   & 0.21 & 0.040   &   0.040    & {\bf 0.035}  \\
& & & \\
\hline
\end{tabular}
\end{adjustbox}
\caption{Results for AIS dataset. In bold the best values. }
\label{tab:ais_results}
\end{table}

The methods called VCMM (k-means) used in \cite{Vine_SAHIN2022} for this dataset, initialized with k-means  reaches the best Miss-classification  value if in our model we set the number of bins equal to $\lceil n^{1/3}\rceil$, however the CopMixM\_BSHQI with random initialization, shows a competitive result and in particular when the Rice Rule is adopted, it achieves the best results. 
 This discrepancy in performance between the two initialization methods indicates that random initialization produces more favorable results for CopMixM\_BSHQI in the context of the AIS dataset.

These results generally highlight the sensitivity of the mixture model algorithm to the choice of initialization method. While the algorithm CopMixM\_BSHQI generally works well, it is critical to consider the impact of initialization on its effectiveness. The higher miss-classification rate observed with K-Means initialization should be taken into account when implementing CopMixM\_BSHQI in scenarios similar to the AIS dataset and its further investigation will be object of future work especially as it seems to improve by setting a specific number of bins. Indeed, by choosing the Rice rule, the CopMixM\_BSHQI results to be very competitive with respect to both initializations.

\subsubsection{Breast Cancer Wisconsin (Diagnostic)}
 The Breast Cancer Wisconsin (Diagnostic) dataset from the UCI Machine Learning Repository  \cite{breast_canc95} consists of digitized images of fine needle aspirates from breast masses from 569 patients. For each of the considered ten features, the mean value, extreme value (mean of the three largest values), and standard error are computed, resulting in 30 total features. The dataset comprises of benign (352 patients) and malignant (212 patients) diagnoses, enabling the measurement of miss-classification rates for binary classification algorithms. In this study, the same features as in \cite{Vine_SAHIN2022} were considered, specifically: perimeter standard error (PSE), extreme smoothness (ES), extreme concavity (EC), and extreme concave points (ECP).

The copMxM\_BSHQI algorithm was applied with both random initialization and K-Means and by setting two different numbers for the bins. Our results are reported in the Table \ref{tab:Breast_results}, and a comparison is made with the outcomes presented in the study in \cite{Vine_SAHIN2022}.
\begin{table}[htbp]
\centering
\begin{adjustbox}{max width=1\textwidth}
\begin{tabular}{lccccccccc}
\hline
\hline
& K-Means&\makecell{VCMM\\ k-means} & \makecell{VCMM\\ (C-vine)} & \makecell{Multivariate\\ normal} & \makecell{Multivariate\\ skew normal} & \makecell{Multivariate\\ t} & \makecell{Multivariate\\ skew t}  &\makecell{CopMixM\_BSHQI\\k-means} &\makecell{CopMixM\_BSHQI\\random} \\
\hline 
\hline
& & & & & & & &\\
\makecell{Miss-classification\\ rate Bins=$\lceil n^{1/3}\rceil$}     & 0.14  &    0.10 &    0.18& 0.12 & 0.15    &   0.11 & 0.15 & {\bf  0.09} &   0.10  \\

\makecell{Miss-classification\\ rate (Bins=Rice) } & 0.14  &    0.10 &    0.18& 0.12 & 0.15    &   0.11 & 0.15  &  {\bf 0.08}& {\bf 0.08}\\
& & & & & & &\\
\hline
\end{tabular}
\end{adjustbox}
\caption{Results for Breast Cancer Dataset. In bold the best values.}
\label{tab:Breast_results}
\end{table}
 The results indicate that CopMixM\_BSHQI with K-Means and with the two choices of bins  
 outperforms the other approaches, achieving a low miss-classification rate of 0.09 and 0.084, respectively. In this case, with the random initialization we can observe  a miss-classification rate $0.082$  only by choosing the Rice bins. 
 This indicate a high degree of accuracy in grouping data points, demonstrating the capability of our methodologies to discern underlying patterns within the dataset.
In contrast, the other methods from \cite{Vine_SAHIN2022}, such as VCMM, VCMM (C-vine), Multivariate normal, Multivariate skew normal, Multivariate t, and Multivariate skew t, exhibit comparatively higher miss-classification rates ranging from 0.10 to 0.18. 
The results show the effectiveness of the CopMixM\_BSHQI approach, when combined with either K-Means or random initialization.

\subsubsection{Text Clustering}
 For the last experiment, we focus on the well-known 20newsgroups dataset, accessible through Scikit-Learn's API and designed for linguistic analysis. This dataset contains online discussions, or newsgroups, categorized into different thematic groups. For our clustering analysis, we specifically selected texts related to technology, religion, and sports.

To convert the textual data into a format suitable for quantitative analysis, we employ the TfidfVectorizer (implemented in \texttt{scikit-learn} library in Python\footnote{\url{https://scikit-learn.org/stable/modules/generated/sklearn.feature_extraction.text.TfidfVectorizer.html}}). This essential tool facilitates the transformation of text into numerical vectors by calculating Term Frequency-Inverse Document Frequency (TF-IDF) values for each term. TF-IDF reflects the importance of each term within individual documents relative to the entire dataset, allowing us to capture semantic information.

It is crucial to note that the TfidfVectorizer is also pivotal for the clustering process. We transform textual data into numerical representations, creating a dataset with $18,846$ rows and $24,471$ features.  Considering the big size of the dataset, a preliminary step in our clustering analysis involves the application of a dimensionality reduction technique. Specifically, we employ Truncated Singular Value Decomposition (T-SVD)\footnote{\url{https://scikit-learn.org/stable/modules/generated/sklearn.decomposition.TruncatedSVD.html}} with a predefined number of components set to 2. This preliminary step is fundamental to mitigate the challenges posed by high-dimensional data, allowing for a more efficient and manageable representation of the underlying structure.

We present scatter plots for GMM, CopMixM\_BSHQI and CopMixM\_KDEpy in Figure \ref{fig:text}. Additionally, the Table \ref{tab:text_results} provides a comparative overview of clustering performance between K-Means, GMM,   CopMixM\_BSHQI and CopMixM\_KDEpy algorithms applied to the 20newsgroup dataset. Again the proposed procedure achieves the best results with respect to Silhouette score and Davies-Bouldin Score, while is the runner up for the Calinski-Harabasz score, but we observe a small discrepancy with the best result provided by K-Means.
\begin{figure}[htbp]
		\centering
		\subfigure[GMM]{\includegraphics[width=5cm]{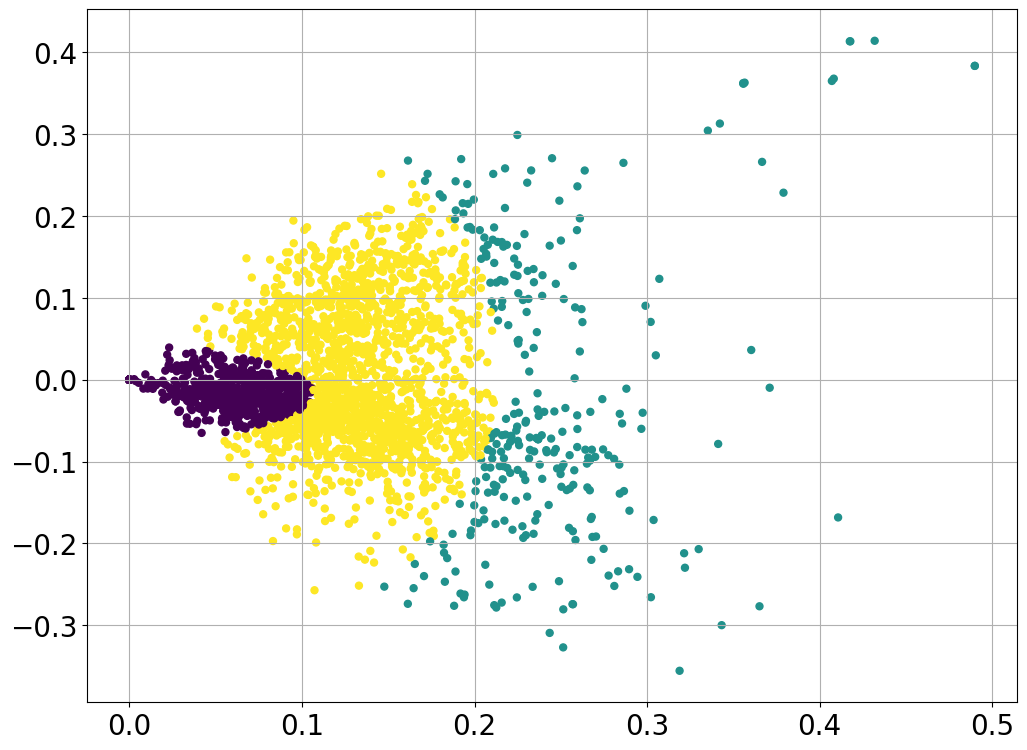}}\quad
		\subfigure[CopMixM\_BSHQI]{\includegraphics[width=5cm]{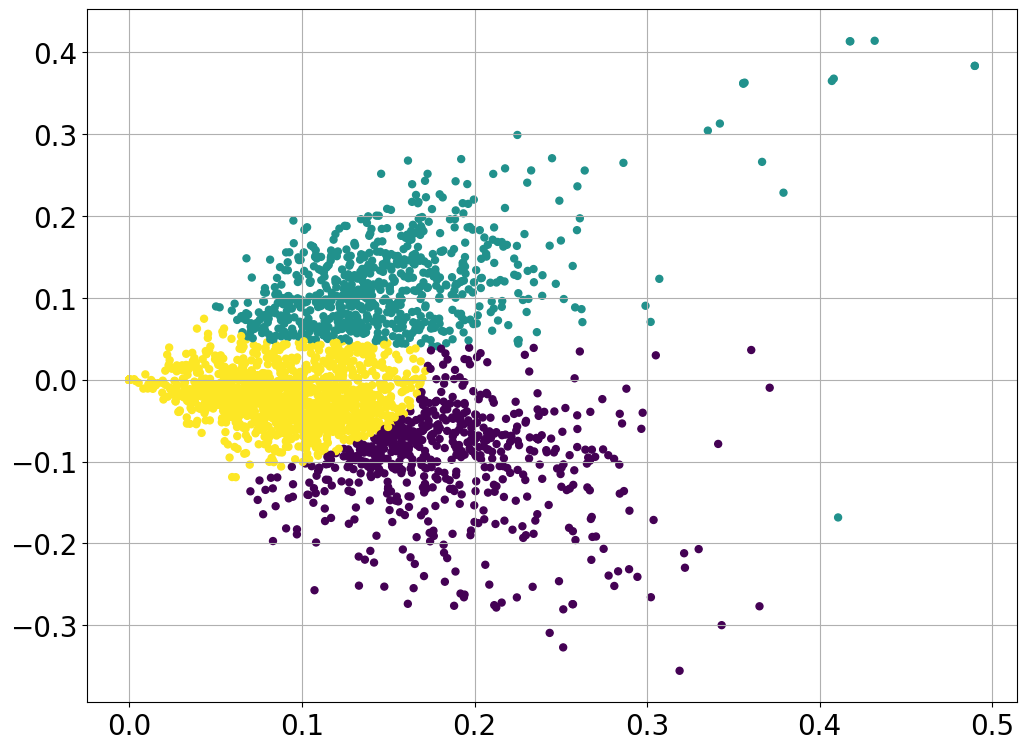}}\quad
  		\subfigure[CopMixM\_KDEpy]{\includegraphics[width=5cm]{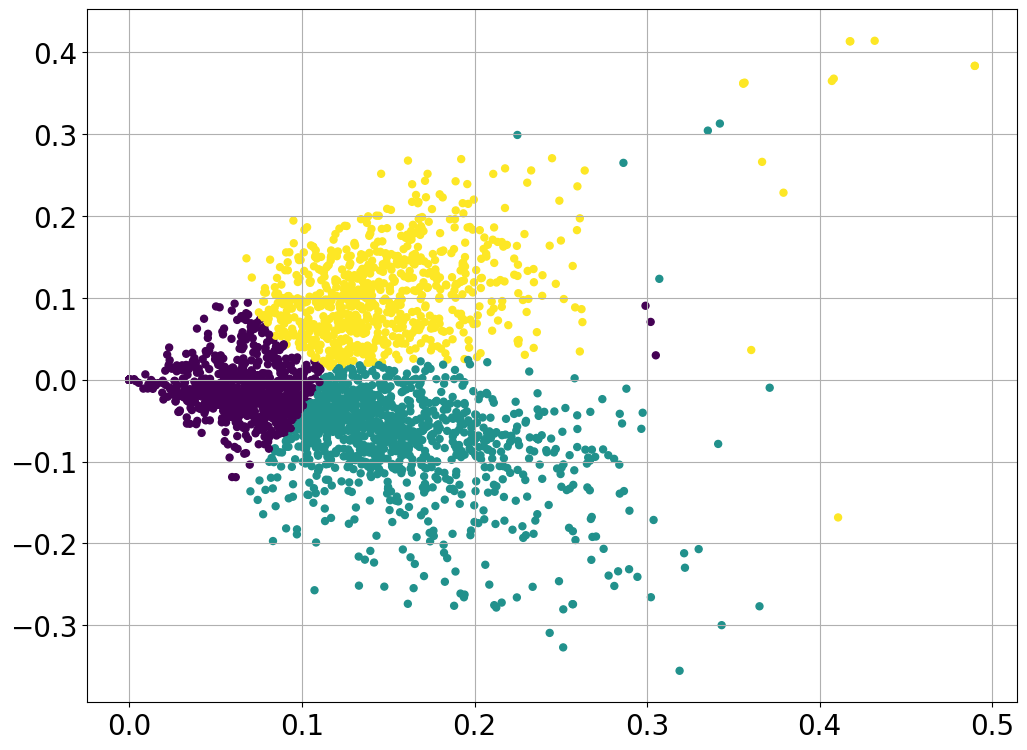}}\quad
		\caption{Text Clustering results.}
		\label{fig:text}
\end{figure}

\begin{table}[htbp]
\centering
 \begin{adjustbox}{max width=1\textwidth}
\begin{tabular}{lcccc}
\hline
\hline
& K-Means & GMM & CopMixM\_BSHQI &CopMixM\_KDEpy \\
\hline
 Silhouette Score    &  0.423  &   0.154  &        {\bf 0.425}  &  0.373  \\
 Calinski-Harabasz Score & {\bf 2326} & 386      &      2200   &        1900        \\
 Davies-Bouldin Score  & 0.896 &   1.89   &        {\bf 0.849}  &     0.856        \\
 \hline
\hline
\end{tabular}
\end{adjustbox}
\caption{Clustering metrics for 20newsgroup Dataset. In bold the best values.}
\label{tab:text_results}
\end{table}

The highest  Silhouette Score emphasizes a better ability to create well-defined clusters similarly, the highest Davis-Boudin score assesses a better compactness and separation between clusters.
The Calinski-Harabasz Score reflects the effectiveness in achieving a favorable ratio of between-cluster to within-cluster variance. 
In summary, across multiple clustering metrics, CopMixM\_BSHQI outperforms GMM, with random initialization, as highlighted by the values in bold and provides a correct interpretation of the undrlying statistical distribution. 

\section{Conclusions}\label{sub:conclusion}

In this paper we presented a novel algorithm for empirical density estimation and we used it for cluster modeling based on the use of Copulas. In particular, the multivariate copulas distribution rely on the estimation of the marginal distributions based on the Hermite Quasi-interpolant in \cite{Mazzia2009}. The proposed construction is superior in terms of statistical significance with respect to classical approaches based on empirical kernel density estimation and provides  consistent cumulative distribution functions as outlined in the detailed analysis of Section \ref{sub:qispline}. The novel clustering algorithm allows for the automatic selection of the best copula among a certain set of copulas families and provides a robust strategy with respect to a random seed as initialization. Moreover, the obtained results show a rather good agreement with the ground-truth (when provided), and mostly we are able to correctly identify the underlying statistical distribution from where the given points where drawn. The obtained clusters exhibit good shape parameters, in terms of Silhouette Score, Calinski-Harabasz Score and Davis-Bouldin Score, and  achieve good accuracy in terms of permutation invariant metrics. Future work will be devoted to investigate the choice of the optimal bandwidth and further to deeply analyze how to deal with overlapping clusters, like in the text mining example.

\section*{Acknowledgments}
Cristiano Tamborrino and Francesca Mazzia acknowledge the support of the PNRR project FAIR - Future AI Research
(PE00000013), Spoke 6 - Symbiotic AI (CUP H97G22000210007) under the NRRP MUR program funded by the NextGenerationEU.
The research of Antonella Falini is founded by  PON 
Ricerca e Innovazione 2014-202 FSE REACT-EU, Azione IV.4 “Dottorati e contratti di ricerca su tematiche dell’innovazione”
CUP H95F21001230006.
The authors thank the GNCS for its valuable support under the INdAM-GNCS project CUP E55F22000270001.

\bibliographystyle{siam}
\bibliography{clustcopbib}

\end{document}